%% file: main.tex
\documentclass{article}

\usepackage[accepted]{icml2026}

\usepackage[utf8]{inputenc} 
\usepackage[T1]{fontenc}    
\usepackage{amsfonts}       
\usepackage{nicefrac}       
\usepackage{placeins}
\usepackage{authblk}
\usepackage{hyperref}
\usepackage{url}
\usepackage{microtype}
\usepackage{graphicx}
\usepackage{subfigure}
\usepackage{booktabs} 
\usepackage{enumitem}
\usepackage{caption}
\usepackage{multicol, multirow}
\usepackage[table]{xcolor}
\usepackage{makecell}
\usepackage{xspace}
\usepackage{fontawesome5}
\usepackage{pifont}
\usepackage{wrapfig}
\usepackage{stfloats}
\usepackage{listings}

\usepackage{amsmath}
\usepackage{amssymb}
\usepackage{mathtools}
\usepackage{amsthm}
\usepackage{tcolorbox}
\usepackage[capitalize,noabbrev]{cleveref}
\usepackage{algorithm}
\usepackage{algorithmic}
\usepackage{appendix}
\usepackage{titletoc}

\theoremstyle{plain}
\newtheorem{theorem}{Theorem}[section]
\newtheorem{proposition}[theorem]{Proposition}

\theoremstyle{definition}
\newtheorem{definition}[theorem]{Definition}

\theoremstyle{remark}

\definecolor{YaleBlue}{rgb}{0.059,0.302,0.573}
\definecolor{forestgreen}{rgb}{0.133,0.549,0.133}
\definecolor{lightblue}{rgb}{0.796, 0.894, 0.9808}
\definecolor{crimson}{rgb}{0.863,0.078,0.235}
\definecolor{natureblue}{RGB}{0,76,153}
\definecolor{naturepurple}{RGB}{115,65,130}
\definecolor{naturegray}{RGB}{248,248,248} 
\definecolor{code_old}{rgb}{0.25,0.5,0.5}
\definecolor{code_new}{rgb}{0.85, 0.18, 0.50}

\input{math_commands.tex}

\newcommand{\cmark}{\textcolor{forestgreen!80!black}{\ding{51}}}
\newcommand{\xmark}{\textcolor{crimson!80!black}{\ding{55}}}

\newcommand{\repeatword}[2]{%
  \begingroup
  \count0=0
  \loop\ifnum\count0<#2
    #1\,
    \advance\count0 by 1
  \repeat
  \endgroup
}

\renewcommand\Affilfont{\normalfont\small} 
\makeatletter
\renewcommand\AB@affilsepx{ \quad\protect\Affilfont\quad } 
\makeatother

\icmltitlerunning{Dispersion Loss Counteracts Embedding Condensation and Improves Generalization in Small Language Models}

\begin{document}

\twocolumn[
\icmltitle{Dispersion Loss Counteracts Embedding Condensation and\\Improves Generalization in Small Language Models}

\icmlsetsymbol{equal}{*}

\begin{icmlauthorlist}
\vskip -0.08in
\icmlauthor{Chen Liu}{equal,yale}
\icmlauthor{Xingzhi Sun}{equal,yale}
\icmlauthor{Xi Xiao}{equal,uab,ornl}
\icmlauthor{Alexandre Van Tassel}{equal,yale}
\icmlauthor{Ke Xu}{yale}
\icmlauthor{Kristof Reimann}{yale}
\icmlauthor{Danqi Liao}{yale}
\icmlauthor{Mark Gerstein}{yale}
\icmlauthor{Tianyang Wang}{uab}
\icmlauthor{Xiao Wang}{ornl}
\icmlauthor{Smita Krishnaswamy}{yale}
\end{icmlauthorlist}

\icmlaffiliation{yale}{Yale University}
\icmlaffiliation{uab}{University of Alabama at Birmingham}
\icmlaffiliation{ornl}{Oak Ridge National Laboratory}

\icmlcorrespondingauthor{Smita Krishnaswamy}{smita.krishnaswamy@yale.edu}

\vskip 4pt
\centering
\small{
\faGithub~~\href{https://github.com/KrishnaswamyLab/LM-Dispersion}{$\texttt{KrishnaswamyLab/LM-Dispersion}$}
} \hspace{8pt}
\faHome~~\href{https://chenliu-1996.github.io/projects/LM-Dispersion/}{Project Page}
\vskip 0.15in
]
\printAffiliationsAndNotice{\icmlEqualContribution}

\begin{abstract}
\vspace{-1.5pt}
Large language models (LLMs) achieve remarkable performance through ever-increasing parameter counts, but scaling incurs steep computational costs. To better understand LLM scaling, we study representational differences between LLMs and their smaller counterparts, with the goal of replicating the representational qualities of larger models in smaller models. We observe a geometric phenomenon which we term \textit{\textbf{embedding condensation}}, where token embeddings collapse into a narrow cone-like subspace in some language models. Through systematic analyses across multiple Transformer families, we show that small models such as \texttt{GPT2} and \texttt{Qwen3-0.6B} exhibit severe condensation, whereas larger models such as \texttt{GPT2-xl} and \texttt{Qwen3-32B} are more resistant to this phenomenon. Additional observations show that embedding condensation is not reliably mitigated by knowledge distillation from larger models. To fight against it, we formulate a dispersion loss that explicitly encourages embedding dispersion during training. Experiments demonstrate that it mitigates condensation, recovers dispersion patterns seen in larger models, and yields performance gains across 10 benchmarks. We believe this work offers a principled path toward improving smaller Transformers without additional parameters.
\end{abstract}

\vspace{-16pt}

\begin{figure*}[!tb]
\centering
\includegraphics[width=\linewidth]{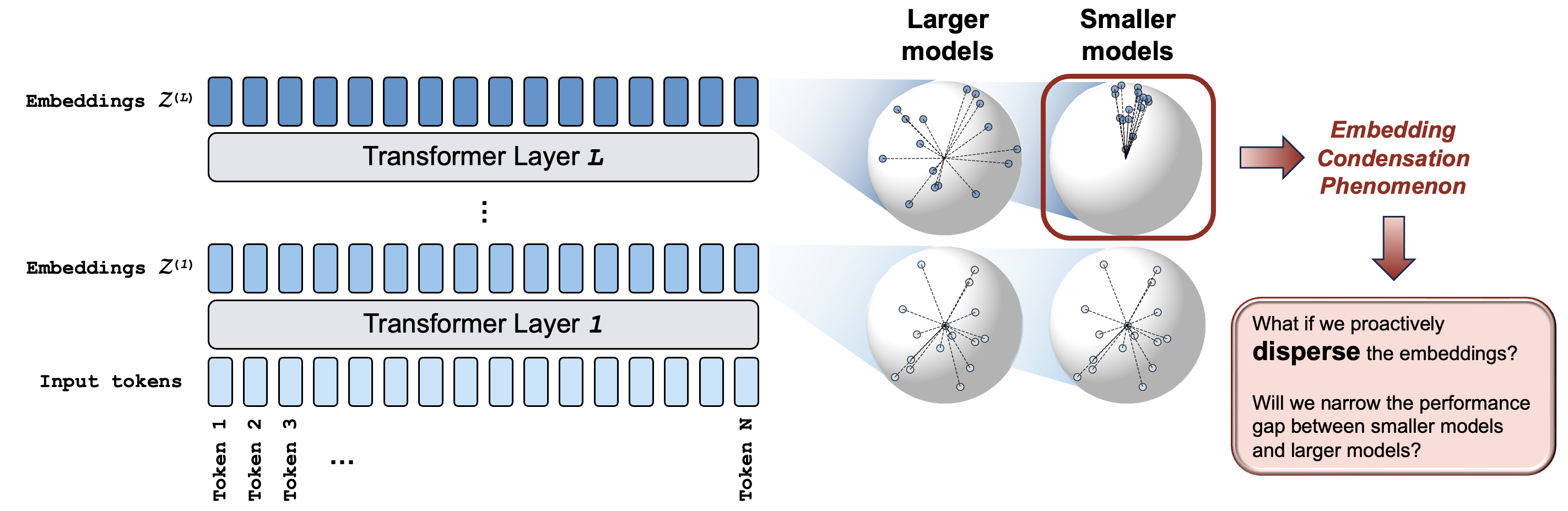}
\caption{Illustration of the embedding condensation phenomenon. In pre-trained language models, embeddings of all tokens from the same input sequence condense into a narrow cone after being processed by many Transformer layers. This phenomenon is substantially more pronounced in smaller models than in larger models within the same family, which motivates our hypothesis in Section~\ref{sec:hypothesis}.}
\label{fig:motivation}
\vspace{-8pt}
\end{figure*}

\section{Introduction}
\label{sec:intro}

The remarkable success of large language models has fundamentally transformed natural language processing, with performance consistently improving as parameter counts scale from millions to trillions~\cite{kaplan2020scaling, hoffmann2022training, minaee2024large}. However, this scaling presents significant practical challenges: larger models require substantial computational resources~\cite{zhang2022opt, llama3, openai2025gpt5}, making them inaccessible for many applications. This motivates a critical question: \textit{Can we identify and replicate the key properties that make large models effective, thus improving smaller models without simply adding more parameters?}

Recent theoretical work on idealized models has shown that Transformer embeddings tend to cluster toward a single point as depth approaches infinity~\cite{Math_Transformer}, but the empirical manifestation of this phenomenon and its relationship to model performance remain underexplored. In this work, we provide a comprehensive empirical analysis of what we term \textit{\textbf{embedding condensation}}.

\vspace{-4pt}
\begin{tcolorbox}[
    colback=natureblue!10,
    boxrule=1pt,
    arc=8pt,
    left=4pt,
    right=4pt,
    top=0pt,
    bottom=0pt
]
\begin{definition}
\label{def:embedding_condensation}
We define \textbf{embedding condensation} as the phenomenon that pairwise cosine similarities of token embeddings in Transformer models concentrate near $1$, implying that embedding vectors point towards nearly identical directions and therefore condense into a narrow cone in the representation space (Figure~\ref{fig:motivation}).
\end{definition}
\end{tcolorbox}
\vspace{-4pt}

Through systematic similarity-based measurements of embedding vector directions across multiple Transformer families, we demonstrate that smaller models (e.g., \texttt{GPT2}, \texttt{Qwen3-0.6B}) exhibit severe embedding condensation, with embedding vectors concentrated towards nearly the same direction and therefore undermining representational diversity. In contrast, larger models (e.g., \texttt{GPT2-xl}, \texttt{Qwen3-32B}) naturally maintain \textit{\textbf{embedding dispersion}}. 

\vspace{-4pt}
\begin{tcolorbox}[
    colback=natureblue!10,
    boxrule=1pt,
    arc=8pt,
    left=4pt,
    right=4pt,
    top=0pt,
    bottom=0pt
]
\begin{definition}
\label{def:embedding_dispersion}
We define \textbf{embedding dispersion} as the resistance to embedding condensation, with embedding vectors pointing towards diverse directions, showing better coverage of the representation space.
\end{definition}
\end{tcolorbox}
\vspace{-4pt}

This geometric perspective reveals a fundamental insight: \textit{condensation might be a key bottleneck limiting the expressivity of smaller Transformers}. Notably, we observe that condensation emerges very early in the training process and cannot be easily resolved by knowledge distillation from larger models, motivating the need for mechanisms that explicitly target embedding geometry.

We hypothesize that the embedding dispersion of larger models leads to their superior performance, suggesting that counteracting condensation could narrow the performance gap between smaller and larger models.

To test this hypothesis, we propose a \textbf{dispersion loss} that explicitly encourages embedding dispersion during training, serving as an auxiliary objective that promotes representational diversity. Our empirical evaluations show that dispersion loss counteracts embedding condensation in smaller models and leads to performance gains across 10 language understanding tasks when applied to models in the \texttt{GPT2} and \texttt{Qwen3} families during mid-training.

Crucially, when incorporated into full pre-training, the proposed dispersion loss also yields an $+1.17$ average improvement across tasks, which achieves a $3.3\%$ gain over the baseline trained with the default cross-entropy loss.

The key contributions of this work are listed below.
\begin{enumerate}[topsep=4pt, itemsep=6pt, parsep=0pt]
    \item We observe and define the \textit{embedding condensation} phenomenon, where cosine similarities between token embeddings concentrate towards $1$ after being processed by Transformer layers.
    
    \item We show that embedding condensation is more pronounced in smaller models than in larger models, emerges at initialization, and is not mitigated by standard knowledge distillation.

    \item We propose a dispersion loss and three alternative formulations that explicitly regulate embedding geometry during training, with stable implementations designed for practical and scalable optimization.

    \item We demonstrate that dispersion-aware training counteracts embedding condensation and improves model generalization, yielding consistent gains during mid-training and during full pre-training.
\end{enumerate}

\section{Methods}
\subsection{Preliminaries: a theoretical perspective on embedding condensation}
Consider a sequence of $N$ tokens and let $\mathcal{Z}^{(l)} = [z_1^{(l)}, z_2^{(l)}, \ldots, z_N^{(l)}]^\top \in \mathbb{R}^{N \times d}$ denote the token embeddings after layer $l$ in a Transformer. These token embeddings are the contextualized representations at each token position, in the form of $d$-dimensional vectors. $\mathcal{Z}^{(l)}$ can be interpreted as $N$ particles in a $d$-dimensional space, and Transformer layers are external impacts on the particle system. A theory paper~\cite{Math_Transformer} has mathematically proven that, under idealized settings, these embeddings tend to cluster into a single point as the number of layers approaches infinity, but limited empirical evidence has been provided.

\begin{figure*}[!th]
\centering
\includegraphics[width=\linewidth]{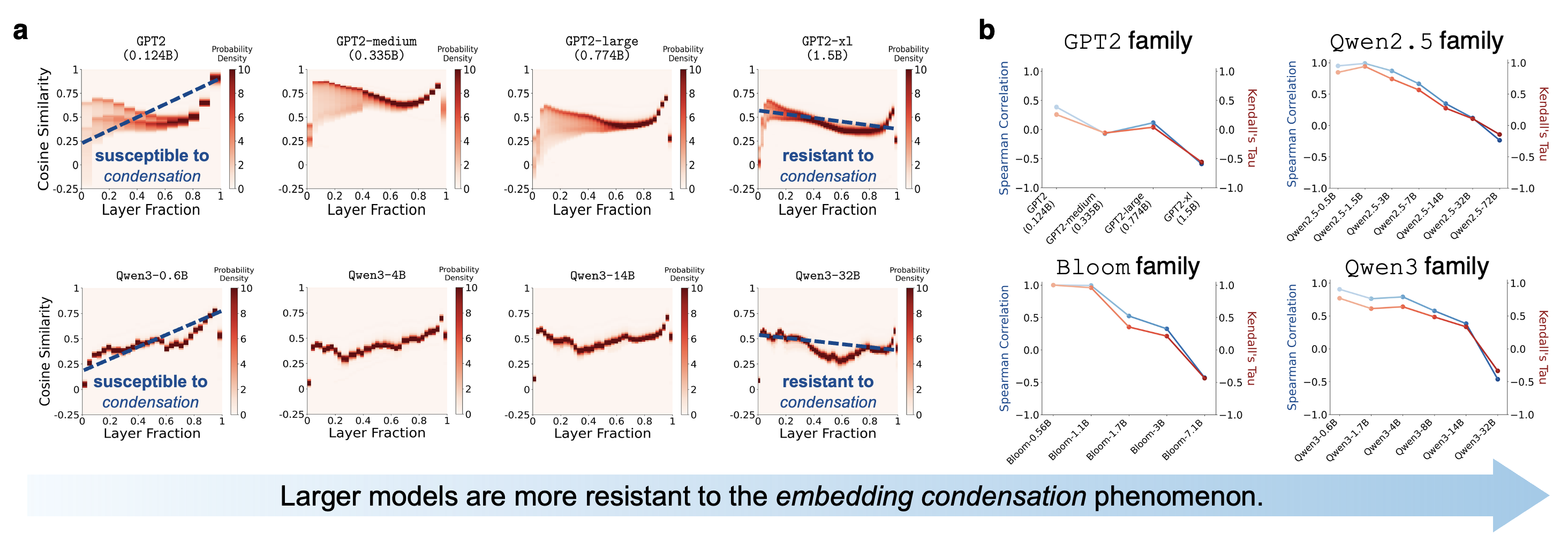}
\caption{Qualitative and quantitative observations of the embedding condensation phenomenon. \textbf{a.} The cosine similarity heatmaps demonstrate that smaller models (e.g., \texttt{GPT2}, \texttt{Qwen3-0.6B}) are susceptible to condensation, since token cosine similarities become increasingly positive as the embeddings proceed to deeper layers. In contrast, larger models (e.g., \texttt{GPT2-xl}, \texttt{Qwen3-32B}) are more resistant to embedding condensation. \textbf{b.} Quantifications using Spearman correlation and Kendall's Tau demonstrate a consistent trend of ``larger model, less condensation'' across multiple families of language models. Additional results can be found in Figure~\ref{fig:supp_observation}.}
\label{fig:observation}
\end{figure*}

\subsection{Quantifying the layer-by-layer evolution of embedding vector alignment in Transformers}

Let $z_i^{(l)} \in \mathbb{R}^d$ denote the embedding of token $i$ after layer $l$. Pairwise cosine similarity quantifies the angle between two embedding vectors, as defined as $\texttt{cossim} \left(z_i^{(l)}, z_j^{(l)} \right) = \frac{z_i^{(l)\top} \cdot z_j^{(l)} }{\lVert z_i^{(l)} \rVert \cdot \lVert z_j^{(l)} \rVert}$. Cosine similarities lie in $[-1, 1]$, with a value of $1$ indicating complete directional alignment, $-1$ indicating opposite directions, and $0$ indicating orthogonality.

To obtain the embeddings at each layer $l$, we feed the input sequence of length $N$ tokens to the Transformer, and collect the token embeddings $[z_1^{(l)}, z_2^{(l)}, \ldots, z_N^{(l)}]^\top$ after the sequence has been processed by the model layers $1,2,\ldots,l$. We then compute cosine similarities $\{\texttt{cossim}(z_i^{(l)}, z_j^{(l)})\}$ for all $N^2$ pairs. The resulting values form a distribution that we visualize as a histogram for each layer. By stacking these histograms across depth, we create a heatmap that highlights the evolution of embedding vector alignment layer by layer. 

In this work, every heatmap is created using a population average over $n=100$ randomly selected input sequences from \texttt{wikitext-103}~\cite{wikitext}. We have experimented with different types of input text corpora, including \texttt{pubmed\_qa}~\cite{pubmed_qa}, \texttt{imdb}~\cite{imdb}, and \texttt{squad}~\cite{squad}, and the same trends persist regardless of the dataset (see Appendix~\ref{sec:condensation_other_datasets}).

\subsection{Comparing the layer-by-layer evolution of embedding vector alignment across models}

To provide quantitative comparisons among models, we compute two metrics to summarize the overall trend of the embedding condensation phenomenon.

For each Transformer layer $l$, we summarize the pairwise cosine similarity distribution by its mean value
$\mu^{(l)} = \frac{1}{N^2} \sum_{i=1}^{N} \sum_{j=1}^{N} \texttt{cossim}\!\left(z_i^{(l)}, z_j^{(l)}\right)$, and then quantify the monotonic relationship between layer depth and embedding similarity by computing rank-based correlation between $\{\mu^{(l)}\}_{l=1}^{L}$ and the layer index sequence $\{l\}_{l=1}^{L}$. Specifically, we report both the Spearman rank correlation coefficient $\rho$~\cite{Spearman_corr} and Kendall's Tau $\tau$~\cite{Kendall_tau}.

These statistics measure the extent to which embedding similarity changes monotonically with depth, independent of absolute scale or nonlinear distortions. Compared to simpler measures such as average cosine similarity in the last few layers, these rank-based correlation metrics show clearer trends~(see Appendix~\ref{sec:condensation_average_cossim}). Large positive/negative values indicate a monotonic increase/decrease in the directional alignment of the embedding vectors, while values near $0$ indicate no systematic trend.

\vspace{-4pt}
\section{Key Observations}
\subsection{Observation of embedding condensation}
\label{sec:observation}

\subsubsection{Analyses on a variety of architectures}

Applying the above analyses to multiple Transformer families reveals a clear trend that depend on model sizes. As shown in Figure~\ref{fig:observation}, smaller models such as \texttt{GPT2} and \texttt{Qwen3-0.6B} exhibit a \textbf{sharp upward drift} of cosine similarity distributions with depth. The embeddings become increasingly aligned, and in \texttt{GPT2} the distribution collapses almost entirely near $1$, indicating a near-perfect directional alignment. \texttt{Qwen3-0.6B} shows the same tendency, though its collapse remains less extreme. We refer to this degeneracy as embedding condensation (Definition~\ref{def:embedding_condensation}).

In contrast, larger models such as \texttt{GPT2-xl} and \texttt{Qwen3-32B} either maintain relatively moderate cosine similarities across layers or exhibit a gradual decrease following an initial increase, suggesting a stronger resistance to embedding condensation. We refer to this behavior as embedding dispersion (Definition~\ref{def:embedding_dispersion}).

For completeness, we also report a simpler metric, namely the average $\texttt{cossim}$ over the last few layers. It shows the same overall trend, although the effect is less pronounced (see Appendix~\ref{sec:condensation_average_cossim}). We therefore use Spearman's $\rho$ and Kendall's $\tau$ for all subsequent analyses.
\vspace{-8pt}
\begin{tcolorbox}[
    colback=naturegray,
    colframe=natureblue!15,
    title=\textsf{Takeaway 1: Larger model, less condensation.},
    coltitle=black,
    boxrule=1pt,
    arc=8pt,
    left=4pt,
    right=4pt,
    top=0pt,
    bottom=0pt
]
\begin{description}[leftmargin=2em, labelwidth=1.4em, itemsep=0pt, parsep=0pt, topsep=0pt]
  \item[\textbf{Q:}] What phenomenon do we observe?
  \item[\textcolor{natureblue}{\textbf{A:}}] \textcolor{natureblue}{Within the same model family, smaller models exhibit severe embedding condensation, with token embeddings collapsing toward near-parallel directions, while larger models resist this collapse.}
\end{description}
\end{tcolorbox}
\vspace{-4pt}

\begin{figure*}[!th]
\centering
\includegraphics[width=\linewidth]{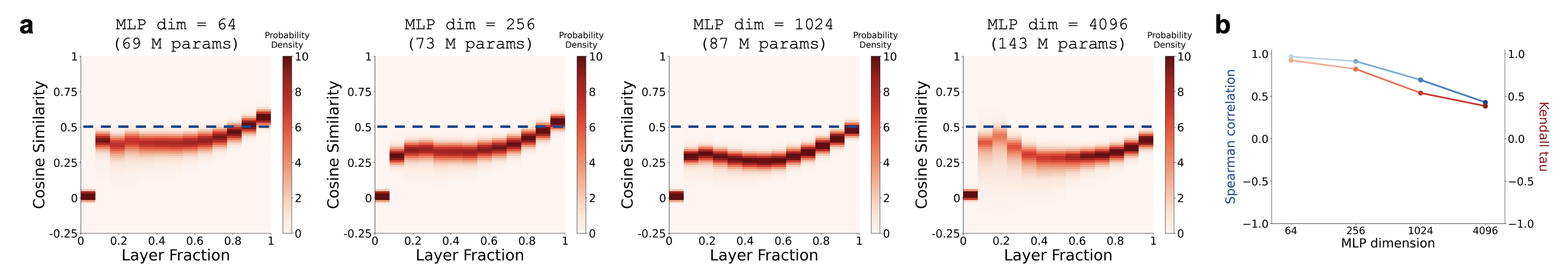}
\caption{In a highly controlled experiment, we reproduced the observation of ``larger model, less condensation''. We pre-trained four \texttt{GPT2}-like models of varying sizes that differ only in MLP dimension, while keeping all other factors fixed, including the number of layers, embedding dimension, dataset, and training configuration. The resulting models exhibit consistent trends in embedding condensation, shown qualitatively (panel \textbf{a}) and quantitatively (panel \textbf{b}). Horizontal dashed lines are added to panel \textbf{a} for easier visual comparison.}
\label{fig:results_controlled_experiment}
\vspace{-8pt}
\end{figure*}

\subsubsection{Confounder-controlled experiment}
\label{sec:controlled_experiment}

To further isolate the effect of model size from other confounding factors, we conduct a controlled experiment in which we pre-train four \texttt{GPT2}-like models for 1B training tokens, varying only the MLP dimension while keeping all other components fixed, including the number of layers, embedding dimension, dataset, and training settings. As shown in Figure~\ref{fig:results_controlled_experiment}, we observe the same trend: larger models consistently exhibit less embedding condensation. This controlled study provides stronger evidence that the relationship between model size and embedding condensation reflects a genuine empirical phenomenon rather than an artifact of correlated confounding factors.

\subsection{Further investigations on embedding condensation}
\label{sec:further_investigations}

We perform additional analyses to better understand when embedding condensation arises and whether it can be alleviated by common training strategies.

\subsubsection{Condensation emerges at initialization and is counteracted during training}
\label{sec:emergence}

We first track the evolution of embedding condensation throughout the pre-training process using checkpoints of \texttt{Olmo-3-1025-7B} spanning initialization, intermediate stages of pre-training, and the final base model (Figure~\ref{fig:observation_training}).

Embedding condensation is the most pronounced at initialization, with both correlation measures taking strongly positive values. This empirical observation is consistent with theoretical results~\cite{Math_Transformer}, which show that condensation arises in Transformers with random $(Q, K, V)$ matrices. As training progresses, these correlations decrease and eventually become negative, indicating that pre-training dynamics counteract, rather than induce, the initial tendency of embedding condensation.

\vspace{-8pt}
\begin{tcolorbox}[
    colback=naturegray,
    colframe=natureblue!15,
    title=\textsf{Takeaway 2: Condensation occurs early on.},
    coltitle=black,
    boxrule=1pt,
    arc=8pt,
    left=4pt,
    right=4pt,
    top=0pt,
    bottom=0pt
]
\begin{description}[leftmargin=2em, labelwidth=1.4em, itemsep=0pt, parsep=0pt, topsep=0pt]
  \item[\textbf{Q:}] At which training stage does the embedding condensation phenomenon occur?
  \item[\textcolor{natureblue}{\textbf{A:}}] \textcolor{natureblue}{It emerges at model initialization and is gradually mitigated, not exacerbated, by pre-training.}
\end{description}
\end{tcolorbox}
\vspace{-4pt}

\begin{figure}[!th]
\centering
\includegraphics[width=0.9\linewidth]{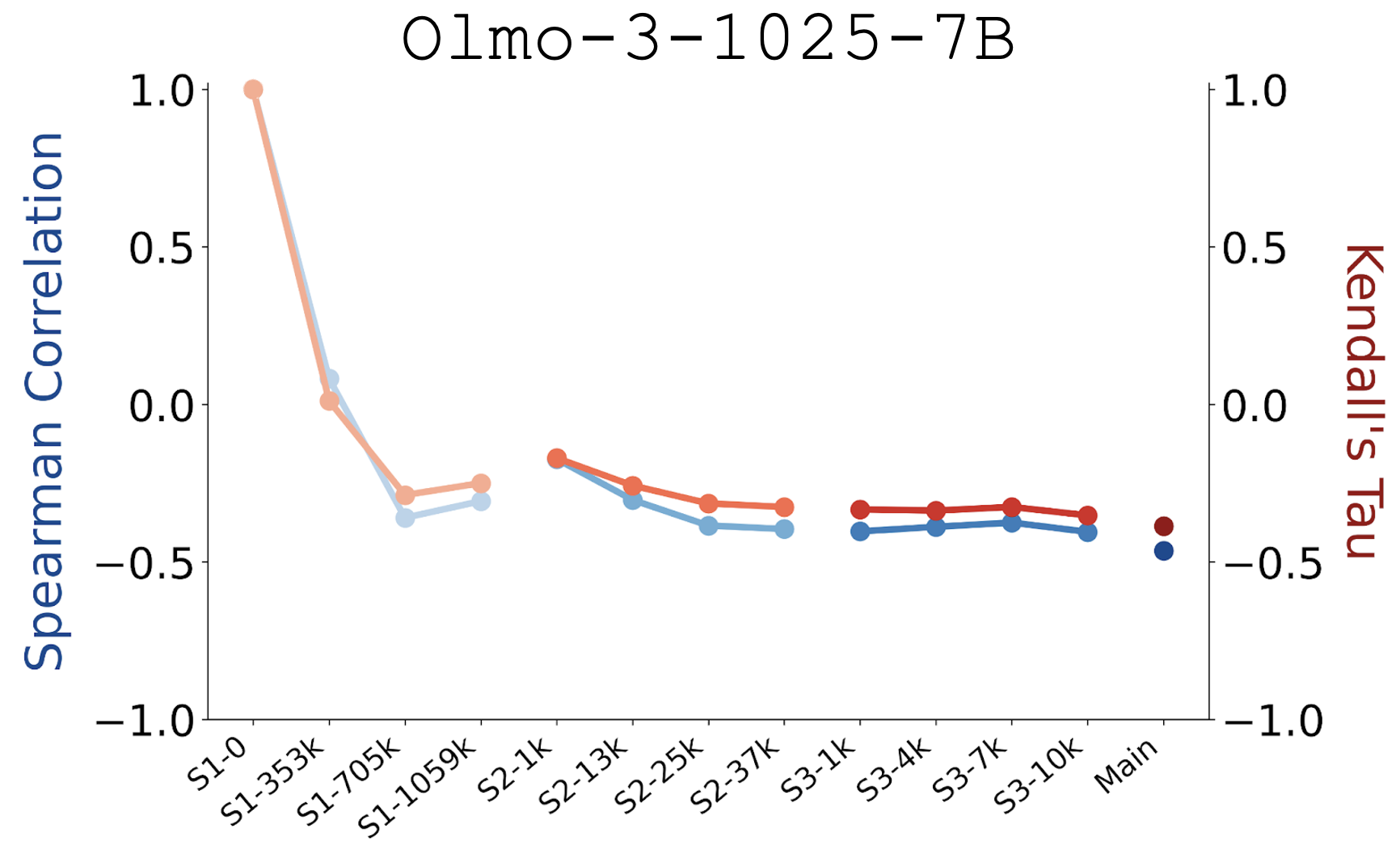}
\caption{Embedding condensation is observed immediately after model initialization. We analyze checkpoints of \texttt{Olmo-3-1025-7B} spanning initialization, intermediate pre-training stages, and the final base model. Each checkpoint is annotated by its training stage and the number of training tokens.}
\label{fig:observation_training}
\vspace{-10pt}
\end{figure}

\begin{figure*}[!th]
\centering
\includegraphics[width=\linewidth]{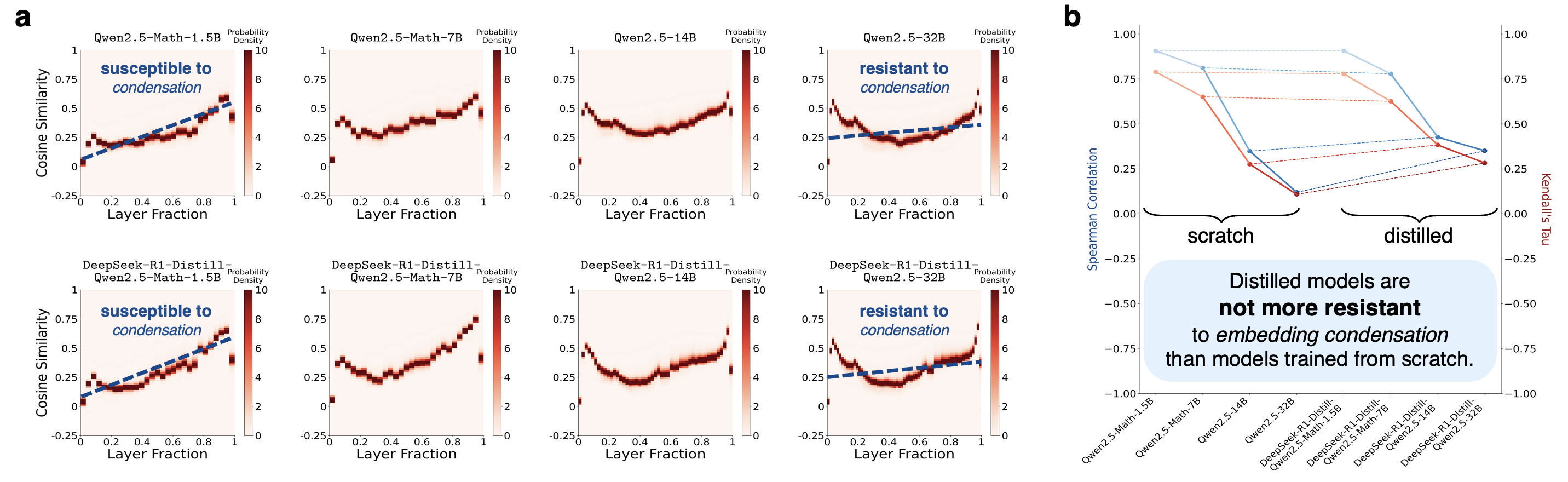}
\vspace{-8pt}
\caption{Knowledge distillation is not a remedy to embedding condensation, shown qualitatively (panel \textbf{a}) and quantitatively (panel \textbf{b}).}
\label{fig:observation_distillation}
\vspace{-4pt}
\end{figure*}

\subsubsection{Knowledge distillation does not inherently mitigate condensation}
\label{sec:knowledge_distillation}

Next, we examine whether knowledge distillation can transfer the favorable embedding geometry of larger models to smaller ones and thus provide a simple remedy to embedding condensation. Using the \texttt{Qwen2.5} family, we compare distilled models with their counterparts trained from scratch across a range of model sizes (Figure~\ref{fig:observation_distillation}).

Distilled models exhibit embedding condensation trends that closely mirror those of non-distilled models, both qualitatively (Figure~\ref{fig:observation_distillation}\textbf{a}) and quantitatively (Figure~\ref{fig:observation_distillation}\textbf{b}). In particular, distillation neither consistently alleviates condensation in smaller models nor amplifies the dispersion behavior characteristic of larger models, indicating that resistance to condensation is \textbf{not automatically inherited} from a larger teacher through distillation (in this case, the teacher model is \texttt{DeepSeek-R1} with 671B parameters). These results motivate the need for explicit mechanisms that directly target embedding geometry during training.

This behavior is expected given the form of the knowledge distillation objective. In modern LLM distillation, the student is trained to match the next-token distribution of the teacher through a logit-level distillation loss. As described in DeepSeek-R1~\cite{DeepSeek}, the distillation loss~\cite{KD_loss_DeepSeek} between the teacher logits $\ell_T^{(i)}$ and the student logits $\ell_S^{(i)}$ at token $i$ is given by \eqref{eqn:kd_loss}.
\begin{align}
\begin{split}
\label{eqn:kd_loss}
\mathcal{L}_{\mathrm{KD}} (\ell_T^{(i)}, \ell_S^{(i)}) 
&=
-\tau^2 \sum_{a=1}^{V}
\sigma_a \left(\frac{\ell_T^{(i)}}{\tau}\right)
\log \sigma_a \left(\frac{\ell_S^{(i)}}{\tau}\right)\\
\sigma_a(\ell)
&=
\frac{\exp(\ell_a)}{\sum_{b=1}^{V} \exp(\ell_b)},
\quad a = 1,\ldots,V.
\end{split}
\end{align}
The student is trained using a weighted combination of this term and the standard next-token prediction loss.

By construction, knowledge distillation primarily constrains the student at the level of output distributions. It does not explicitly regulate intermediate token embeddings, their pairwise relationships, or the layer-wise gradients that shape the representation geometry. Consequently, while knowledge distillation can effectively transfer predictive behavior, it does not inherently control the internal representational dynamics responsible for embedding condensation, explaining why resistance to condensation is not automatically inherited from a larger teacher model.

\vspace{-4pt}
\begin{tcolorbox}[
    colback=naturegray,
    colframe=natureblue!15,
    title=\textsf{Takeaway 3: Distillation is not a solution.},
    coltitle=black,
    boxrule=1pt,
    arc=8pt,
    left=4pt,
    right=4pt,
    top=0pt,
    bottom=0pt
]
\begin{description}[leftmargin=2em, labelwidth=1.4em, itemsep=0pt, parsep=0pt, topsep=0pt]
  \item[\textbf{Q:}] Larger models are resistant to condensation. Can we distill from them to obtain this resistance?
  \item[\textcolor{natureblue}{\textbf{A:}}] \textcolor{natureblue}{No. Knowledge distillation does not transfer the desired resistance to embedding condensation.}
\end{description}
\end{tcolorbox}
\vspace{-4pt}

\begin{figure*}[!th]
\centering
\includegraphics[width=\linewidth]{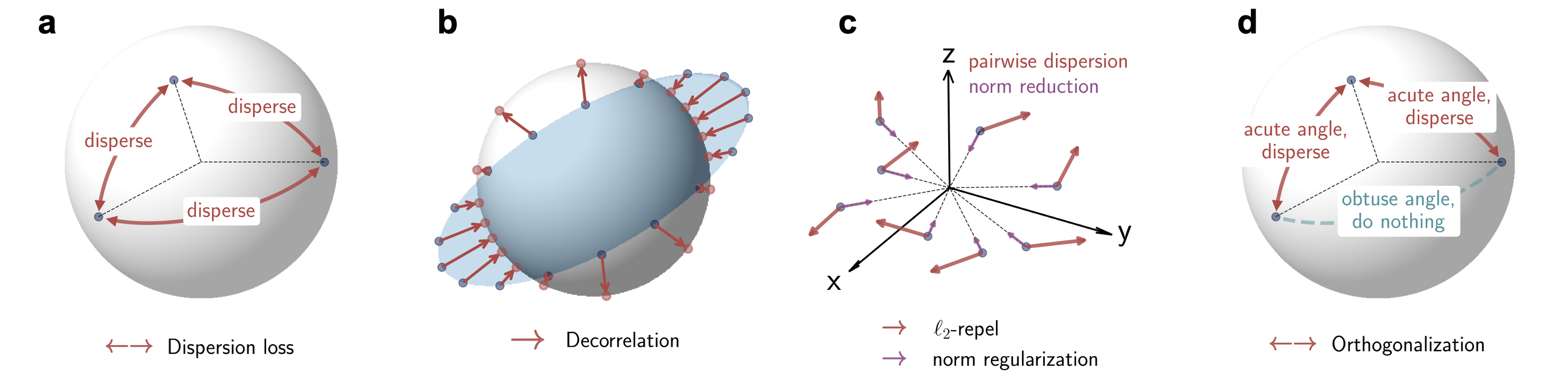}
\caption{Illustration of how dispersion loss and its alternative formulations promote embedding dispersion. \textbf{a.} Dispersion loss enforces uniform angular dispersion by spreading out all pairs along the unit hypersphere. \textbf{b.} Decorrelation loss encourages different feature dimensions to remain uncorrelated. \textbf{c.} $\ell_2$-repel loss increases pairwise Euclidean distance, while the norm regularization prevents unbounded expansion. \textbf{d.} Orthogonalization loss spreads out vectors forming acute angles while leaving obtuse ones unchanged.}
\label{fig:loss_illustration}
\end{figure*}

\begin{table*}[!th]
\caption{Our dispersion loss and its alternative formulations. Main implementation differences from \cite{diffuse_and_disperse} are highlighted in \textcolor{code_old}{teal} and \textcolor{code_new}{magenta}. Including or excluding diagonal terms yields identical gradients and is therefore cosmetic. For dispersion loss and $\ell_2$-repel, we adopt the \texttt{log-sum-exp} trick for numerical stability, which differs from $\log(\mathrm{mean}(\exp(\cdot)))$ only by an additive constant. For $\ell_2$-repel, we include a norm regularization term to prevent unbounded expansion of embeddings. For Orthogonalization, the distance margin is fixed to $\tfrac{1}{2}$ since we use angular distance, where $\tfrac{1}{2}$ corresponds to orthogonality and thus serves as the ideal margin.}
\vspace{-4pt}
\centering
\resizebox{\textwidth}{!}{%
    \begin{tabular}{lcccc}
    \toprule
    & \multicolumn{2}{c}{For generative modeling in} & \multicolumn{2}{c}{\textbf{For improving generalization of}} \\
    & \multicolumn{2}{c}{diffusion-based models~\cite{diffuse_and_disperse}} & \multicolumn{2}{c}{\textbf{Transformer-based language models (Ours)}} \\
    \cmidrule(lr){2-3} \cmidrule(lr){4-5}
    & formulation & term definition & formulation & term definition \\
    \midrule
    \textbf{Dispersion loss}
    & $\log \mathbb{E}_{i,j} [\exp (-D (z_i, z_j)/\tau)]$
    & \textcolor{code_old}{$D(z_i, z_j) = -\texttt{cossim}(z_i, z_j)$}
    & $\log \sum_{i,j}^{i \neq j} [\exp (-D (z_i, z_j)/\tau)]$ 
    & \textcolor{code_new}{$D(z_i, z_j) = \frac{\arccos \left( \texttt{cossim}(z_i, z_j) \right)}{\pi}$} \\
    \midrule
    {\footnotesize \textit{Alternative formulations}} \\
    Decorrelation
    & $\sum_{m,n} \text{Cov}_{mn}^2$
    & $\text{Cov}^2 = \frac{\mathcal{Z}_c^\top \mathcal{Z}_c}{d-1}$, $\mathcal{Z}_c = \frac{\mathcal{Z} - \mu_d(\mathcal{Z})}{\sigma_d(\mathcal{Z})}$
    & $\sum_{m, n}^{m \neq n} \text{Cov}_{mn}^2$ 
    & $\text{Cov}^2 = \frac{\mathcal{Z}_c^\top \mathcal{Z}_c}{d-1}$, $\mathcal{Z}_c = \frac{\mathcal{Z} - \mu_d(\mathcal{Z})}{\sigma_d(\mathcal{Z})}$ \\
    ${\ell_2}$-repel
    & $\log \mathbb{E}_{i,j} [\exp (-D (z_i, z_j)/\tau)]$
    & $D(z_i, z_j) = \lVert \mathcal{Z}_{i,:} - \mathcal{Z}_{:,j} \rVert_2^2 $
    & $\log \sum_{i,j}^{i \neq j} [\exp (-D (z_i, z_j)/\tau)] \textcolor{code_new}{+ \lambda_\text{norm} \lVert\mathcal{Z}\rVert_2^2}$
    & $D(z_i, z_j) = \lVert \mathcal{Z}_{i,:} - \mathcal{Z}_{:,j} \rVert_2^2 $\\
    Orthogonalization
    & $\mathbb{E}_{i,j} [\texttt{max}(0, \textcolor{code_old}{\epsilon} - D(z_i, z_j))^2]$
    & \textcolor{code_old}{$D(z_i, z_j) = -\texttt{cossim}(z_i, z_j)$}
    & $\mathbb{E}_{i, j}^{i \neq j} [\texttt{max}(0, \textcolor{code_new}{\frac{1}{2}} - D(z_i, z_j))^2]$ 
    & \textcolor{code_new}{$D(z_i, z_j) = \frac{\arccos \left( \texttt{cossim}(z_i, z_j) \right)}{\pi}$} \\
    \bottomrule
    \end{tabular}
}
\label{tab:loss_formulation}
\end{table*}

\subsection{Our Hypothesis}
\label{sec:hypothesis}

The observations above highlight an important implication: condensation reduces the diversity of directions in which tokens can be represented, effectively narrowing the model's expressive capacity. More importantly, we found that it cannot be easily remedied by distillation from a large model.

These observations motivate the following hypothesis.

\vspace{-4pt}
\begin{tcolorbox}[
    colback=naturegray,
    colframe=naturepurple!15,
    title=\textsf{Hypothesis: Larger models are better in language tasks because they counteract condensation.},
    coltitle=black,
    boxrule=1pt,
    arc=8pt,
    left=4pt,
    right=4pt,
    top=0pt,
    bottom=0pt
]
Embedding condensation reduces the expressivity of Transformers by collapsing token embedding vectors into narrow cones, under-utilizing the representation space. We hypothesize that by \textcolor{naturepurple}{\textbf{dispersing embeddings during training}}, smaller models can achieve representational qualities more similar to larger models, thus narrowing the performance gap without increasing the number of parameters.
\end{tcolorbox}

\subsection{Our Remedy: Dispersion Loss}
\label{sec:dispersion_loss}
Our hypothesis motivates the design of auxiliary objectives that explicitly promote embedding dispersion during training. For this purpose, we propose to augment the training loss with a dispersion loss as a regularizer. The dispersion loss is given by \eqref{eqn:dispersion_loss} and the full training objective is given by \eqref{eqn:main_loss}. The pseudocode is provided in Appendix~\ref{sec:pseudo_code}. Here, $\mathcal{L}_\text{train}$ denotes the standard training loss, which defaults to the cross-entropy loss for next-token prediction in most language models.
\vspace{-4pt}
\begin{align}
\label{eqn:dispersion_loss}
\mathcal{L}_\text{disp} &= \log \sum\nolimits_{i,j}^{i \neq j} e^{-\frac{\arccos \left( \texttt{cossim}(z_i, z_j) \right)}{\pi \tau}}\\
\label{eqn:main_loss}
\mathcal{L} &= \mathcal{L}_{\text{train}} + \lambda_\text{disp} \cdot \mathcal{L}_{\text{disp}}
\end{align}
\paragraph{Dispersion loss}
\vspace{-8pt}
The dispersion loss is a straightforward objective that directly counteracts the condensation of cosine similarities by spreading out all embeddings on the unit hypersphere (Table~\ref{tab:loss_formulation} row 1 and Figure~\ref{fig:loss_illustration}a). In practice, we use the inverse cosine to map cosine similarity to an angular distance for numerical stability. During training, the loss is computed over the embedding vectors of each input sequence and aggregated across all layers. The resulting implementation has a time complexity of $\mathcal{O}(N^2 F)$ per batch, where $N$ is the sequence length and $F$ is the feature dimension. In practice, we can subsample a small number of tokens over the sequence dimension, and the additional computational cost becomes manageable.

In addition to the canonical dispersion loss, we implemented three alternative formulations of the dispersion loss, and evaluated them in our main experiments.

\paragraph{Decorrelation}
The decorrelation formulation minimizes off-diagonal entries of the covariance matrix of embeddings~(Table~\ref{tab:loss_formulation} row 2 and Figure~\ref{fig:loss_illustration}b). By construction, this loss reduces the correlations between feature dimensions, which indirectly promotes more diverse embedding vector directions in the representation space.

\paragraph{$\ell_2$-repel}
The $\ell_2$-repel formulation directly pushes pairs of embedding vectors apart in the Euclidean space. However, minimizing this objective can be achieved trivially by increasing the embedding norms, as larger magnitudes inflate pairwise distances. To prevent this degeneracy, we include an explicit norm regularization term that constrains unbounded growth (Table~\ref{tab:loss_formulation} row~3 and Figure~\ref{fig:loss_illustration}c).

\paragraph{Orthogonalization}
The orthogonalization formulation is similar to the canonical dispersion loss, except that the dispersion vanishes when two vectors are orthogonal to each other~(Table~\ref{tab:loss_formulation} row 4 and Figure~\ref{fig:loss_illustration}d). The distance margin $\epsilon$ is naturally set to $\frac{1}{2}$, which corresponds to orthogonality under the angular distance.

\subsubsection{Potential effectiveness on larger models}

Having introduced dispersion loss as an explicit mechanism for regulating embedding geometry, we revisit the role of model size in embedding condensation. As shown in Appendix~\ref{sec:dimension}, increased embedding dimensionality provides a geometric expectation in which randomly oriented vectors show reduced condensation, but this expectation does not guarantee that trained representations fully utilize the available space. As a result, the resistance to embedding condensation empirically observed in larger models may reflect an increased representational capacity rather than an explicit dispersion mechanism. This observation raises the possibility that \textit{our dispersion loss could benefit not only small models but also large models}, which we leave for future investigations.

\vspace{-4pt}
\section{Empirical Results}
We evaluate our proposed dispersion loss under two training regimes. First, we conduct mid-training experiments~\cite{Octothinker}, in which we continue training pre-trained \texttt{GPT2} and \texttt{Qwen3} models for an additional 200~M tokens on the \texttt{wikitext-103} dataset~\cite{wikitext}. We repeat the experiments under three random seeds, and report the mean and standard deviation for each metric. This setting provides a computationally efficient proof of concept, feasible on a single NVIDIA A100 GPU, enabling controlled ablations and systematic hyperparameter studies.

We then perform full pre-training from scratch to examine the effects of incorporating dispersion loss on the formation of representational geometry. \texttt{Qwen3} models are trained on the \texttt{allenai/c4} dataset~\cite{AI2_C4} for 156~B tokens using 640 GPUs.

Experimental details, including training protocols and hyperparameters, are provided in Appendix~\ref{sec:setting_and_hyperparams}. The evaluation benchmarks are described in Appendix~\ref{sec:benchmark_description}.

\begin{figure*}[!th]
    \centering
    \includegraphics[width=\linewidth]{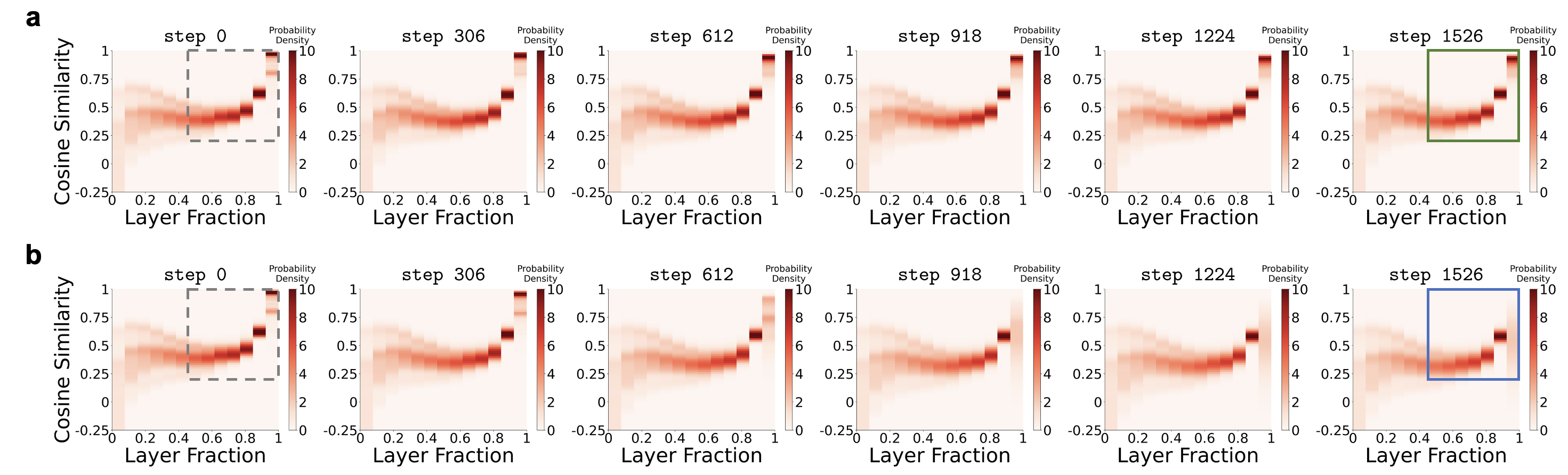}
    \caption{Dispersion loss counteracts the embedding condensation phenomenon. \textbf{a.} Starting from condensed embeddings (gray dashed box), mid-training with the default loss has a limited impact (green box). \textbf{b.} In contrast, mid-training with our dispersion loss as a regularizer substantially mitigates embedding condensation (blue box). }
    \label{fig:results_condensation_counteract}
\vspace{-8pt}
\end{figure*}

\vspace{-4pt}
\subsection{Dispersion loss counteracts the embedding condensation phenomenon}
\label{sec:counteract}

First, we examine whether dispersion loss can directly counteract the embedding condensation phenomenon. 

Using the same heatmap visualization, we observe that pre-trained \texttt{GPT2} exhibits severe embedding condensation, with pairwise cosine similarities rapidly collapsing toward $1$ in deeper layers~(Figure~\ref{fig:results_condensation_counteract} column 1). Continuing training with the standard cross-entropy objective provides minimal relief, leaving the overall condensation pattern largely intact~(Figure~\ref{fig:results_condensation_counteract}a). In contrast, incorporating dispersion loss substantially alters the geometry of token representations. As training progresses, the cosine similarity distributions become more spread out~(Figure~\ref{fig:results_condensation_counteract}b). These results indicate that dispersion-aware training can restore representational diversity even when applied during mid-training.

\begin{table*}[!th]
\caption{Using dispersion loss during mid-training improves performance on language tasks. For each base model, the best single-benchmark metrics are displayed in bold, whereas the best average ranks and best average performances are boxed and in bold. We also perform the Student's $t$-tests on the average performances and report the significance level with respect to ``$\mathcal{L}_\text{train}$ + Dispersion loss''.}
\vspace{-4pt}
\centering
\setlength{\tabcolsep}{2pt}
\resizebox{\textwidth}{!}{%
    \begin{tabular}{lclc ccccccccccccc}
    \toprule
    \multirow{2}{*}{Model} & \multicolumn{2}{c}{Mid-training} & Training time & \multicolumn{6}{c}{Zero-shot} & \multicolumn{4}{c}{Few-shot} & \multirow{2}{*}{\textbf{Rank}$\downarrow$} & \multirow{2}{*}{\textbf{Average}$\uparrow$} & \multirow{2}{*}{$t$-test} \\
    
    \cmidrule(r){2-3} \cmidrule(lr){4-4} \cmidrule(lr){5-10} \cmidrule(lr){11-14}
    & Train & \multicolumn{1}{c}{Loss} & A100 hours
    & \texttt{ANLI$_\text{R2}$}$\uparrow$
    & \texttt{LAMBADA$_\text{openai}$}$\uparrow$
    & \texttt{OpenbookQA}$\uparrow$
    & \texttt{PIQA}$\uparrow$
    & \texttt{TruthfulQA}$\uparrow$
    & \texttt{WinoGrande}$\uparrow$
    & \texttt{ARC$_\text{easy}$}$\uparrow$
    & \texttt{ARC$_\text{challenge}$}$\uparrow$
    & \texttt{MedMCQA}$\uparrow$
    & \texttt{MMLU}$\uparrow$ \\

    \midrule
    \texttt{GPT2} & \xmark & \multicolumn{1}{c}{---} & \multicolumn{1}{c}{---} & 
    34.4 & 30.8 & 15.6 & 61.6 & 40.3 & 52.4 & 42.0 & 16.2 & 25.4 & 24.8 & 6.1 & 34.35 & $p$ < 0.0001 \\

    \cmidrule(l){2-16}
    & \cmark & $\mathcal{L}_\text{train}$ & 1.122 (1.00$\times$) &
    34.0 $\pm$ 0.3 & 32.4 $\pm$ 0.6 & 16.4 $\pm$ 0.1 & 62.2 $\pm$ 0.7 & 43.6 $\pm$ 0.2 & 52.8 $\pm$ 3.1 & 41.6 $\pm$ 0.0 & 17.4 $\pm$ 0.6 & 24.2 $\pm$ 3.8 & 24.8 $\pm$ 0.4 & 6.2 & 34.95 $\pm$ 0.11 & $p$ < 0.01 \\

    & \cmark & $\mathcal{L}_\text{train}$ + noisy embedding & 1.122 (1.00$\times$) &
    34.3 $\pm$ 0.1 & 33.0 $\pm$ 0.3 & 16.9 $\pm$ 0.4 & 61.1 $\pm$ 0.7 & 43.9 $\pm$ 0.2 & 53.6 $\pm$ 2.0 & 43.7 $\pm$ 0.4 & 18.0 $\pm$ 0.6 & 21.6 $\pm$ 2.5 & 25.4 $\pm$ 0.4 & 4.3 & 35.15 $\pm$ 0.06 & $p$ < 0.01 \\

    & \cmark & $\mathcal{L}_\text{train}$ + active forgetting & 1.127 (1.00$\times$) &
    \textbf{35.0} $\pm$ 0.3 & 33.8 $\pm$ 1.1 & 16.6 $\pm$ 0.3 & 61.0 $\pm$ 0.0 & 44.0 $\pm$ 0.1 & 52.2 $\pm$ 0.8 & \textbf{44.3} $\pm$ 0.4 & \textbf{18.6} $\pm$ 1.1 & 22.4 $\pm$ 1.7 & \textbf{25.7} $\pm$ 0.2 & \textbf{\fbox{3.2}} & 35.36 $\pm$ 0.15 & \textit{n.s.} \\

    \cmidrule(l){2-16}
    & \cmark & $\mathcal{L}_\text{train}$ + {\small Decorrelation} & 1.221 (1.09$\times$)
    & \textbf{35.0} $\pm$ 0.0 & 33.6 $\pm$ 0.5 & \textbf{17.0} $\pm$ 0.6 & 60.8 $\pm$ 0.6 & 43.9 $\pm$ 0.3 & 52.0 $\pm$ 1.7 & 43.8 $\pm$ 0.5 & 18.0 $\pm$ 1.3 &  24.4 $\pm$ 2.3 & 25.6 $\pm$ 0.2 & 3.6 & 35.41 $\pm$ 0.06 & \textit{n.s.} \\

    & \cmark & $\mathcal{L}_\text{train}$ + {\small $\ell_2$-repel} & 1.175 (1.05$\times$) & 34.6 $\pm$ 0.3 & 32.8 $\pm$ 0.2 & 16.6 $\pm$ 0.2 & 60.8 $\pm$ 0.3 & 43.9 $\pm$ 0.2 & \textbf{54.2} $\pm$ 0.8 & 43.4 $\pm$ 0.2 & 17.6 $\pm$ 0.8 & 24.6 $\pm$ 2.0 & \textbf{25.7} $\pm$ 0.1 & 4.1 & 35.42 $\pm$ 0.11 & \textit{n.s.} \\

    & \cmark & $\mathcal{L}_\text{train}$ + {\small Orthogonalization} & 1.210 (1.08$\times$) & 34.4 $\pm$ 0.1 & 31.2 $\pm$ 1.1 & 16.8 $\pm$ 0.7 & 61.6 $\pm$ 0.0 & \textbf{44.2} $\pm$ 0.4 & 53.4 $\pm$ 0.8 & 44.2 $\pm$ 1.0 & 18.0 $\pm$ 0.3 & 24.8 $\pm$ 3.3 & 25.6 $\pm$ 0.2 & \textbf{\fbox{3.2}} & 35.42 $\pm$ 0.19 & \textit{n.s.} \\

    \rowcolor{lightblue!60}\cellcolor{white} & \cellcolor{white}\cmark & $\mathcal{L}_\text{train}$ + Dispersion loss & 1.176 (1.05$\times$)
    & 34.8 $\pm$ 0.0 & \textbf{34.0} $\pm$ 0.4 & 16.6 $\pm$ 0.0 & \textbf{62.4} $\pm$ 0.3 & \textbf{44.2} $\pm$ 0.1 & 51.6 $\pm$ 0.1 & 42.6 $\pm$ 0.4 & 18.0 $\pm$ 0.8 & \textbf{26.6} $\pm$ 2.3 & 25.4 $\pm$ 0.2 & \textbf{\fbox{3.2}} & \textbf{\fbox{35.61}} $\pm$ 0.12 & \textit{ref.} \\ 

    \midrule
    \texttt{GPT2-m} & \xmark & \multicolumn{1}{c}{---} & \multicolumn{1}{c}{---} & 33.3 & 40.4 & 18.6 & 66.3 & 40.1 & 54.8 & 50.2 & 19.9 & \textbf{29.1} & 25.3 & 5.4 & 37.80 & $p$ < 0.0001 \\

    \cmidrule(l){2-16}
    & \cmark & $\mathcal{L}_\text{train}$ & 2.541 (1.00$\times$) & 33.1 $\pm$ 0.6 & 43.2 $\pm$ 0.3 &  19.1 $\pm$ 0.0 & 67.7 $\pm$ 0.4 & 40.5 $\pm$ 1.3 & 55.5 $\pm$ 1.0 & \textbf{54.2} $\pm$ 2.7 &  18.9 $\pm$ 0.7 & 28.2 $\pm$ 2.3 & 25.1 $\pm$ 0.3 & 4.5 & 38.55 $\pm$ 0.18 & $p$ < 0.01 \\

    & \cmark & $\mathcal{L}_\text{train}$ + noisy embedding & 2.580 (1.01$\times$) &
    33.3 $\pm$ 0.1 & 44.0 $\pm$ 0.3 & 18.6 $\pm$ 0.6 & 66.7 $\pm$ 0.1 & \textbf{44.5} $\pm$ 0.6 & 52.1 $\pm$ 1.6 & 51.7 $\pm$ 1.6 & 19.5 $\pm$ 1.8 & 25.5 $\pm$ 0.1 & 25.8 $\pm$ 0.3 & 4.8 & 38.16 $\pm$ 0.25 & $p$ < 0.01 \\

    & \cmark & $\mathcal{L}_\text{train}$ + active forgetting & 2.580 (1.01$\times$) &
    33.4 $\pm$ 0.3 & 43.0 $\pm$ 0.6 & 18.9 $\pm$ 0.1 & 66.7 $\pm$ 0.1 & 44.3 $\pm$ 0.1 & 50.8 $\pm$ 1.4 & 51.8 $\pm$ 2.5 & 19.8 $\pm$ 1.4 & 26.5 $\pm$ 0.1 & \textbf{26.1} $\pm$ 0.0 & 4.4 & 38.13 $\pm$ 0.29 & $p$ < 0.01 \\

    \cmidrule(lr){2-16}
    & \cmark & $\mathcal{L}_\text{train}$ + {\small Decorrelation} & 2.760 (1.09$\times$) & 33.4 $\pm$ 0.0 & \textbf{45.4} $\pm$ 1.3 & 18.9 $\pm$ 0.8 & 66.6 $\pm$ 0.1 & 42.2 $\pm$ 0.1 & 56.2 $\pm$ 0.4 & 53.8 $\pm$ 3.3 & 18.0 $\pm$ 0.8 & 28.2 $\pm$ 3.0 & 25.4 $\pm$ 0.2 & 4.0 & 38.81 $\pm$ 0.12 & $p$ < 0.01 \\

    & \cmark & $\mathcal{L}_\text{train}$ + {\small $\ell_2$-repel} & 2.675 (1.05$\times$) & \textbf{33.6} $\pm$ 0.8 & 44.2 $\pm$ 0.8 & 18.5 $\pm$ 0.6 & 66.2 $\pm$ 0.1 & 42.4 $\pm$ 0.1 & 54.4 $\pm$ 1.1 & 54.0 $\pm$ 2.0 & 18.6 $\pm$ 0.6 & 28.9 $\pm$ 2.1 & 25.3 $\pm$ 0.4 & 4.7 & 38.61 $\pm$ 0.03 & $p$ < 0.01 \\

    & \cmark & $\mathcal{L}_\text{train}$ + {\small Orthogonalization} & 2.692 (1.06$\times$) & 
    33.0 $\pm$ 0.4 & 45.2 $\pm$ 0.4 & 18.6 $\pm$ 0.6 & \textbf{68.1} $\pm$ 0.4 &  41.4 $\pm$ 0.2 & 55.1 $\pm$ 0.6 & 53.6 $\pm$ 2.7 & 18.4 $\pm$ 0.4 & 29.0 $\pm$ 0.4 & 25.0 $\pm$ 0.1 & 4.8 & 38.74 $\pm$ 0.30 & $p$ < 0.05 \\

    \rowcolor{lightblue!60}\cellcolor{white} & \cellcolor{white}\cmark & $\mathcal{L}_\text{train}$ + Dispersion loss & 2.673 (1.05$\times$) & \textbf{33.6} $\pm$ 0.1 & 45.2 $\pm$ 0.8 & \textbf{19.2} $\pm$ 0.2 & 67.5 $\pm$ 0.1 & 43.4 $\pm$ 0.2 & \textbf{56.4} $\pm$ 1.1 & 53.8 $\pm$ 1.1 & \textbf{20.1} $\pm$ 0.9 & 28.6 $\pm$ 0.1 & 25.7 $\pm$ 0.3 & \textbf{\fbox{2.2}} & 
    \textbf{\fbox{39.35}} $\pm$ 0.15 & \textit{ref.} \\

    \midrule
    \texttt{GPT2-l} & \xmark & \multicolumn{1}{c}{---} & ---
    & 33.6 & 47.8 & 19.6 & 71.6 & 38.9 & 58.4 & 54.0 & 22.4 & 26.2 & 25.6 & --- & 39.81 & --- \\

    \texttt{GPT2-xl} & \xmark & \multicolumn{1}{c}{---} & ---
    & 36.4 & 49.4 & 22.2 & 71.8 & 38.0 & 57.2 & 58.0 & 23.6 & 27.0 & 25.2 & --- & 40.89 & --- \\
    
    \specialrule{0.6pt}{6pt}{6pt}

    \texttt{Qwen3-0.6B} & \xmark & \multicolumn{1}{c}{---} & ---
    & 35.0 & 43.0 & 19.5 & 66.5 & 40.7 & 60.5 & 67.5 & 32.5 & 26.5 & 49.4 & 5.7 & 44.11 & $p$ < 0.0001 \\

    \cmidrule(l){2-16}
    & \cmark & $\mathcal{L}_\text{train}$ & 4.676 (1.00$\times$) &
    32.5 $\pm$ 0.3 & \textbf{52.0} $\pm$ 0.3 & 21.5 $\pm$ 0.5 & \textbf{67.5} $\pm$ 1.4 & 44.3 $\pm$ 0.9 & \textbf{61.0} $\pm$ 0.8 & 68.0 $\pm$ 2.8 & 33.0 $\pm$ 1.3 & 29.5 $\pm$ 0.9 & \textbf{50.0} $\pm$ 0.1 & 4.1 & 45.93 $\pm$ 0.37 & $p$ < 0.01 \\

    & \cmark & $\mathcal{L}_\text{train}$ + noisy embedding & 4.839 (1.03$\times$) &
    34.0 $\pm$ 0.7 & 48.8 $\pm$ 0.4 & 20.5 $\pm$ 0.7 & 66.0 $\pm$ 0.7 & 49.4 $\pm$ 0.1 & 57.5 $\pm$ 0.7 & 67.2 $\pm$ 1.1 & 34.5 $\pm$ 0.7 & 35.0 $\pm$ 1.4 & 48.8 $\pm$ 0.0 & 5.3 & 46.17 $\pm$ 0.28 & $p$ < 0.01 \\

    & \cmark & $\mathcal{L}_\text{train}$ + active forgetting & 4.829 (1.03$\times$) & 
    35.0 $\pm$ 0.7 & 49.0 $\pm$ 4.2 & 20.0 $\pm$ 2.1 & 67.2 $\pm$ 0.4 & 48.6 $\pm$ 0.9 & 58.8 $\pm$ 1.8 & 69.0 $\pm$ 1.4 & 34.2 $\pm$ 1.1 & 36.8 $\pm$ 0.4 & 49.3 $\pm$ 0.0 & 4.0 & 46.79 $\pm$ 0.27 & $p$ < 0.05 \\

    \cmidrule(lr){2-16}
    & \cmark & $\mathcal{L}_\text{train}$ + {\small Decorrelation} & 4.939 (1.06$\times$) & 
    35.0 $\pm$ 0.0 & 50.5 $\pm$ 0.4 & 19.5 $\pm$ 1.8 & \textbf{67.5} $\pm$ 1.1 & 47.0 $\pm$ 2.3 & 59.5 $\pm$ 1.8 & 68.5 $\pm$ 3.2 & \textbf{35.0} $\pm$ 1.8 & 34.0 $\pm$ 0.4 & 49.8 $\pm$ 0.2 & 3.5 & 46.62 $\pm$ 0.25 & $p$ < 0.01 \\

    & \cmark & $\mathcal{L}_\text{train}$ + {\small $\ell_2$-repel} & 4.864 (1.04$\times$) & 
    33.5 $\pm$ 0.0 & 46.0 $\pm$ 2.1 & 19.0 $\pm$ 0.7 & 66.0 $\pm$ 0.4 & 47.4 $\pm$ 0.2 & 59.5 $\pm$ 1.4 & 69.0 $\pm$ 3.2 & \textbf{35.0} $\pm$ 1.8 & \textbf{40.0} $\pm$ 0.4 & 46.9 $\pm$ 0.2 & 4.8 & 46.23 $\pm$ 0.04 & $p$ < 0.001  \\

    & \cmark & $\mathcal{L}_\text{train}$ + {\small Orthogonalization} & 4.936 (1.06$\times$) & 
    \textbf{35.5} $\pm$ 1.1 & 50.0 $\pm$ 0.7 & 22.0 $\pm$ 0.7 & 65.5 $\pm$ 0.4 & 49.1 $\pm$ 0.1 & 56.5 $\pm$ 2.1 & 71.5 $\pm$ 3.2 & 33.0 $\pm$ 1.1 & 36.0 $\pm$ 0.7 & 48.9 $\pm$ 0.2 & 4.2 & 46.80 $\pm$ 0.20 & $p$ < 0.05  \\

    \rowcolor{lightblue!60}\cellcolor{white} & \cellcolor{white}\cmark & $\mathcal{L}_\text{train}$ + Dispersion loss & 4.878 (1.04$\times$) &
    \textbf{35.5} $\pm$ 0.8 & 49.5 $\pm$ 0.8 & \textbf{22.5} $\pm$ 0.6 & 65.0 $\pm$ 1.6 & \textbf{49.8} $\pm$ 1.1 & 58.5 $\pm$ 1.6 & \textbf{72.5} $\pm$ 2.3 & 34.5 $\pm$ 0.8 & 37.5 $\pm$ 1.3 & 49.2 $\pm$ 0.2 & \textbf{\fbox{3.2}} & \textbf{\fbox{47.45}} $\pm$ 0.16 & \textit{ref.} \\ 

    \midrule
    \texttt{Qwen3-1.7B} & \xmark & \multicolumn{1}{c}{---} & ---
    & 39.5 & 53.0 & 29.0 & 71.5 & 47.6 & 60.5 & 72.0 & 50.0 & 45.5 & 63.1 & 5.1 & 53.18 & $p$ < 0.0001 \\

    \cmidrule(l){2-16}
    & \cmark & $\mathcal{L}_\text{train}$ & 9.148 (1.00$\times$) & 40.0 $\pm$ 0.9 & 60.7 $\pm$ 0.3 & 28.5 $\pm$ 0.9 & 74.3 $\pm$ 0.6 & 49.1 $\pm$ 0.2 & 61.3 $\pm$ 0.3 & 71.3 $\pm$ 1.4 & 47.7 $\pm$ 1.2 & 49.7 $\pm$ 1.2 & 63.1 $\pm$ 0.0 & 4.3 & 54.57 $\pm$ 0.17 & $p$ < 0.01 \\

    & \cmark & $\mathcal{L}_\text{train}$ + noisy embedding & 9.345 (1.02$\times$) & 35.0 $\pm$ 0.0 & 60.5 $\pm$ 2.8 & 28.0 $\pm$ 0.0 & 70.5 $\pm$ 0.7 & \textbf{50.3} $\pm$ 0.9 & 60.8 $\pm$ 3.2 & 73.2 $\pm$ 0.4 & 48.5 $\pm$ 0.0 & 48.5 $\pm$ 2.1 & 62.5 $\pm$ 0.6 & 5.0 & 53.78 $\pm$ 0.07 & $p$ < 0.0001 \\

    & \cmark & $\mathcal{L}_\text{train}$ + active forgetting & 9.298 (1.02$\times$) & 35.2 $\pm$ 0.4 & 59.5 $\pm$ 1.4 & 27.5 $\pm$ 0.7 & 72.5 $\pm$ 0.0 & 49.9 $\pm$ 0.2 & 60.5 $\pm$ 1.4 & \textbf{74.2} $\pm$ 0.4 & 47.2 $\pm$ 1.1 & 49.0 $\pm$ 1.4 & 63.2 $\pm$ 0.2 & 5.0 & 53.89 $\pm$ 0.07 & $p$ < 0.001 \\

    \cmidrule(lr){2-16}
    & \cmark & $\mathcal{L}_\text{train}$ + {\small Decorrelation} & 9.535 (1.04$\times$) & 39.5 $\pm$ 1.5 & 60.7 $\pm$ 1.0 & 28.3 $\pm$ 1.2 & 74.8 $\pm$ 1.0 & 49.7 $\pm$ 0.5 & 61.5 $\pm$ 0.9 & 72.7 $\pm$ 1.5 & 49.7 $\pm$ 0.8 & 49.7 $\pm$ 1.0 & 63.5 $\pm$ 0.2 & 3.1 & 55.01 $\pm$ 0.13 & $p$ < 0.01 \\

    & \cmark & $\mathcal{L}_\text{train}$ + {\small $\ell_2$-repel} & 9.440 (1.03$\times$) & 33.5 $\pm$ 0.0 & 47.8 $\pm$ 0.6 & 23.0 $\pm$ 0.0 & 67.8 $\pm$ 1.2 & 42.1 $\pm$ 0.2 & 57.5 $\pm$ 0.5 & 68.3 $\pm$ 1.9 & 41.3 $\pm$ 1.8 & 38.5 $\pm$ 2.6 & 42.3 $\pm$ 0.5 & 8.0 & 46.22 $\pm$ 0.57 & $p$ < 0.0001 \\

    & \cmark & $\mathcal{L}_\text{train}$ + {\small Orthogonalization} & 9.565 (1.05$\times$) & \textbf{40.8} $\pm$ 1.3 & 61.0 $\pm$ 0.5 & \textbf{29.2} $\pm$ 0.3 & 74.8 $\pm$ 1.3 & 49.5 $\pm$ 0.2 & 61.3 $\pm$ 0.8 & 72.3 $\pm$ 1.9 & 48.2 $\pm$ 1.6 & 49.5 $\pm$ 1.3 & 63.7 $\pm$ 0.0 & 3.0 & 55.04 $\pm$ 0.26 & $p$ < 0.05 \\

    \rowcolor{lightblue!60}\cellcolor{white} & \cellcolor{white}\cmark & $\mathcal{L}_\text{train}$ + Dispersion loss & 9.455 (1.03$\times$) & 
    40.2 $\pm$ 2.0 & \textbf{62.2} $\pm$ 0.8 & 29.0 $\pm$ 1.0 & \textbf{75.2} $\pm$ 0.8 & 50.2 $\pm$ 0.4 & \textbf{62.7} $\pm$ 1.2 & 73.0 $\pm$ 1.3 & \textbf{49.3} $\pm$ 0.8 & \textbf{51.0} $\pm$ 1.8 & \textbf{64.1} $\pm$ 0.2 & \textbf{\fbox{1.7}} & \textbf{\fbox{55.68}} $\pm$ 0.21 & \textit{ref.} \\ 

    \midrule
    \texttt{Qwen3-4B} & \xmark & \multicolumn{1}{c}{---} & ---
    & 41.5 & 60.5 & 27.5 & 75.0 & 52.8 & 67.0 & 80.0 & 58.5 & 55.5 & 73.1 & --- & 59.14 & --- \\

    \texttt{Qwen3-8B} & \xmark & \multicolumn{1}{c}{---} & ---
    & 49.5 & 67.5 & 33.0 & 77.5 & 54.7 & 71.0 & 86.0 & 63.5 & 56.0 & 78.2 & --- & 63.68 & --- \\

    \texttt{Qwen3-14B} & \xmark & \multicolumn{1}{c}{---} & ---
    & 54.0 & 65.5 & 35.5 & 77.0 & 54.7 & 74.5 & 86.0 & 68.5 & 62.0 & 81.7 & --- & 65.93 & --- \\
    \bottomrule

    \end{tabular}
}
\label{tab:results_midtraining}
\vspace{-4pt}
\end{table*}

\begin{table*}[!th]
\caption{Dispersion loss is more effective when applied to deeper layers, where embedding condensation is more pronounced. Experiments are performed under the \texttt{GPT2} mid-training setting.}
\vspace{-4pt}
\centering
\setlength{\tabcolsep}{4pt}
\resizebox{\textwidth}{!}{%
    \begin{tabular}{cc cccccccccccccc}
    \toprule
    & \multicolumn{6}{c}{Zero-shot} & \multicolumn{4}{c}{Few-shot} & \multirow{2}{*}{\textbf{Average}$\uparrow$} & $t$-test \\
    
    \cmidrule(lr){2-7} \cmidrule(lr){8-11}
    \multicolumn{1}{c}{Dispersion location}
    & \texttt{ANLI$_\text{R2}$}$\uparrow$
    & \texttt{LAMBADA$_\text{openai}$}$\uparrow$
    & \texttt{OpenbookQA}$\uparrow$
    & \texttt{PIQA}$\uparrow$
    & \texttt{TruthfulQA}$\uparrow$
    & \texttt{WinoGrande}$\uparrow$
    & \texttt{ARC$_\text{easy}$}$\uparrow$
    & \texttt{ARC$_\text{challenge}$}$\uparrow$
    & \texttt{MedMCQA}$\uparrow$
    & \texttt{MMLU}$\uparrow$ \\

    \midrule
    Layers $\in [1, \frac{L}{2}]$
    & 34.6 $\pm$ 0.3 & 33.0 $\pm$ 0.4 & 16.7 $\pm$ 0.6 & 60.9 $\pm$ 0.4 & 43.3 $\pm$ 0.5 & \textbf{52.4} $\pm$ 0.3 & 43.7 $\pm$ 0.3 & \textbf{17.9} $\pm$ 1.3 & 21.9 $\pm$ 1.0 & 25.3 $\pm$ 0.6 & 34.96 $\pm$ 0.20 & $p$ < 0.05 \\

    Layers $\in [\frac{L}{2}, L]$
    & \textbf{34.9} $\pm$ 0.4 & \textbf{33.3} $\pm$ 0.6 & \textbf{16.9} $\pm$ 0.1 & \textbf{61.3} $\pm$ 0.8 & \textbf{43.9} $\pm$ 0.1 & 52.2 $\pm$ 1.1 & \textbf{43.8} $\pm$ 1.2 & 17.8 $\pm$ 0.9 & \textbf{23.7} $\pm$ 2.4 & \textbf{25.8} $\pm$ 0.1 & \textbf{\fbox{35.37}} $\pm$ 0.06 & \textit{ref.} \\
    \bottomrule
    \end{tabular}
}
\label{tab:results_early_late_layer}
\vspace{-8pt}
\end{table*}

\begin{table}[!tb]
\caption{Effect of hyperparameters on the dispersion loss. Ablation experiments are performed under the \texttt{GPT2} mid-training setting.}
\vspace{-4pt}
\centering
\setlength{\tabcolsep}{2pt}
\resizebox{0.45\textwidth}{!}{%
    \begin{tabular}{cccc}
    \toprule
    \multirow{2}{*}{Loss} & ~~Coefficient~~ & ~~Temperature~~ & \multirow{2}{*}{\textbf{Average} (same 10 tasks) $\uparrow$} \\
    & $\lambda_\text{disp}$ & $\tau$ \\
    \midrule

    \rowcolor{lightblue!60}\cellcolor{white} $\mathcal{L}_\text{train}$ + Dispersion loss & 0.1 & 1.0 & \textbf{\fbox{35.61}} \\
    \cmidrule(l){2-4}
    & 0.01 & 1.0 & 35.36 \\
    & 0.5 & 1.0 & 35.37 \\
    & 1.0 & 1.0 & 35.29 \\
    \cmidrule(l){2-4}
    & 0.1 & 0.1 & 35.33 \\
    & 0.1 & 0.5 & 35.42 \\
    & 0.1 & 2.0 & 35.27 \\
    \bottomrule
    \end{tabular}
}
\label{tab:ablation_midtraining_GPT2}
\vspace{-8pt}
\end{table}

\vspace{-4pt}
\subsection{Dispersion loss is effective in mid-training}
\label{sec:mid_training}

Next, we evaluate whether the geometric improvements induced by the dispersion loss translate into better downstream performance. We report zero-shot and few-shot results on 10 language understanding benchmarks for models before and after mid-training. Results for the \texttt{GPT2} and \texttt{Qwen3} models are reported in Table~\ref{tab:results_midtraining}. Besides our dispersion loss and alternative formulations, we compared against other simple ways that could potentially counteract the embedding condensation phenomenon, namely ``noisy embedding'' which adds noise to the embedding vectors during training, inspired by~\cite{jain2024neftune}, and ``active forgetting'' which resets the token embedding layer periodically during training, inspired by~\cite{active_forgetting}.

Models mid-trained with dispersion loss consistently outperform those trained with the default loss, yielding improvements across most tasks and model sizes. Although the absolute gains are modest, they are systematic and consistent, further supporting the link between reduced condensation and improved generalization. Among all compared methods, the proposed dispersion loss delivers the strongest and most stable performance across tasks, achieving the highest average improvement across multiple model sizes~(Table~\ref{tab:results_midtraining}). Interestingly, active forgetting~\cite{active_forgetting} is also fairly competitive, despite not being originally designed for this purpose.

The three alternative formulations of dispersion loss also yield gains in some settings, but are generally less stable or slightly weaker on average, motivating our focus on the canonical dispersion loss in subsequent experiments. $\ell_2$-repel is the most volatile, likely due to the fragile balance between pairwise dispersion and norm regularization: if the former dominates, the norm explodes; if the latter dominates, the result is collapse rather than dispersion. Decorrelation is less effective than dispersion loss likely because it penalizes feature-channel correlations and is less direct than reducing angular alignment. Orthogonalization is also less flexible, as it does not incentivize angular separation beyond 90 degrees.

\vspace{-4pt}
\paragraph{Differential effect on early and late layers}
To further examine whether dispersion loss improves performance via counteracting embedding condensation, we analyze its layer-wise impact. Motivated by the observation that condensation becomes more pronounced in deeper layers, we apply dispersion loss selectively to either the first half or the last half of the model. We find that applying the loss to the later layers yields greater performance improvements than applying it to the earlier layers (Table~\ref{tab:results_early_late_layer}). This result provides additional evidence that the empirical benefits of dispersion loss arise from mitigating the embedding condensation phenomenon.

\vspace{-4pt}
\paragraph{Ablation studies on hyperparameters}
We then conduct ablation studies to assess the sensitivity of the proposed dispersion loss to its main hyperparameters, namely the weighting coefficient $\lambda_{\text{disp}}$ and the temperature parameter $\tau$, using mid-training experiments on \texttt{GPT2}~(Table~\ref{tab:ablation_midtraining_GPT2}). The average score across the same 10 language understanding benchmarks is reported.

In general, we find that the dispersion loss is relatively robust to the choice of $\lambda_{\text{disp}}$ and $\tau$. Based on these results, we adopt $\lambda_{\text{disp}} = 0.1$ and $\tau = 1.0$ as default settings in subsequent experiments.

\vspace{-4pt}
\subsection{Dispersion loss is effective in pre-training}
\label{sec:pre_training}

Finally, we evaluate the effect of dispersion loss when incorporated throughout full pre-training. Following the insights obtained from the mid-training experiments, we perform pre-training from scratch using the \texttt{Qwen3-0.6B} model on the \texttt{allenai/c4} corpus with 640 GPUs. Pre-training with dispersion loss leads to an average improvement of $+1.17$ points in downstream evaluation metrics, including $+4.0$ points on \texttt{PIQA}, and $+7.4$ points on \texttt{TruthfulQA}, indicating that encouraging embedding dispersion during representation formation is beneficial as we anticipated. These gains are achieved over a diverse set of language understanding tasks without any task-specific post-training, suggesting improved generalization. These results demonstrate that incorporating dispersion loss throughout pre-training provides a principled and effective mechanism for counteracting embedding condensation.

We observe that dispersion is generally more helpful for knowledge-heavy or long-context reasoning tasks like \texttt{TruthfulQA} and \texttt{MedMCQA}. In these tasks, we expect the models to benefit from less collapsed, more distinguishable contextual token directions. In contrast, tasks such as \texttt{WinoGrande} rely more on common sense inference and might benefit less from representation dispersion.

\begin{table*}[!th]
\caption{Using dispersion loss during pre-training improves performance on language tasks. Experiments are performed under the \texttt{Qwen3-0.6B} pre-training setting.}
\vspace{-6pt}
\centering
\setlength{\tabcolsep}{4pt}
\resizebox{\textwidth}{!}{%
    \begin{tabular}{lc cccccccccccccc}
    \toprule
    & \multicolumn{6}{c}{Zero-shot} & \multicolumn{4}{c}{Few-shot} & \multirow{2}{*}{\textbf{Average}$\uparrow$}\\
    
    \cmidrule(lr){2-7} \cmidrule(lr){8-11}
    \multicolumn{1}{c}{Loss}
    & \texttt{ANLI$_\text{R2}$}$\uparrow$
    & \texttt{LAMBADA$_\text{openai}$}$\uparrow$
    & \texttt{OpenbookQA}$\uparrow$
    & \texttt{PIQA}$\uparrow$
    & \texttt{TruthfulQA}$\uparrow$
    & \texttt{WinoGrande}$\uparrow$
    & \texttt{ARC$_\text{easy}$}$\uparrow$
    & \texttt{ARC$_\text{challenge}$}$\uparrow$
    & \texttt{MedMCQA}$\uparrow$
    & \texttt{MMLU}$\uparrow$ \\

\midrule
    $\mathcal{L}_\text{train}$ 
    & \textbf{34.0} & 24.0 & \textbf{15.5} & 64.5 & 37.8 & \textbf{58.0} & \textbf{41.5} & 22.0 & 26.5 & 24.6 & 34.84 \\

    \rowcolor{lightblue!60} $\mathcal{L}_\text{train}$ + Dispersion loss
    & 32.0 & \textbf{27.5} & 13.5 & \textbf{68.5} & \textbf{45.2} & 54.5 & 40.0 & \textbf{24.5} & \textbf{29.5} & \textbf{24.9} & \textbf{\fbox{36.01}}~\textcolor{forestgreen}{$_{(+1.17)}$} \\
    \bottomrule
    \end{tabular}
}
\label{tab:results_pretraining_Qwen3}
\vspace{-16pt}
\end{table*}

\vspace{-6pt}
\section{Related Works}
\paragraph{Analyses of condensation}
Phenomena consistent with what we term embedding condensation have appeared in prior analyses of Transformer representations, though typically in indirect or task-specific forms. Existing studies have characterized related behaviors using a variety of measures, including output matrix rank~\cite{shi2022revisiting}, covariance matrix rank~\cite{li2025tracing}, distance to rank-1 subspaces measured by the Frobenius norm~\cite{dong2021attention}, spectral bias between high- and low-frequency components~\cite{wang2022anti}, singular values~\cite{zhang2025mitigating}, entropy~\cite{liao2024assessing, tokarchuk2026representation}, spherical variance~\cite{tokarchuk2026representation}, and the proportion of variance explained by the principal components~\cite{ethayarajh2019contextual}. These metrics offer complementary views of representation collapse across layers. In this work, we used the layer-by-layer dynamics in pairwise cosine similarity between token embeddings as a direct and interpretable measure for tracking embedding condensation during training. For a broader overview of related representation degeneration phenomena, we refer readers to the survey of~\cite{dovonon2024setting}.

Previous work has attributed this phenomenon to several factors: oversmoothing induced by layer normalization under analogies between Transformers and graph neural networks~\cite{shi2022revisiting}; the effects of specific components such as self-attention and MLPs~\cite{dong2021attention}; the distribution of embeddings at the infinite depth limit~\cite{Math_Transformer}; and the eigenspectrum of the Transformer update~\cite{dovonon2024setting}.

Prior attempts to mitigate these effects have largely focused on architectural or parameterization changes, such as aggregating representations across layers~\cite{shi2022revisiting}, reparameterizing updates via eigendecomposition~\cite{dovonon2024setting}, or rebalancing frequency components~\cite{wang2022anti}. In contrast, our work targets embedding condensation directly through an explicit representation-level regularization objective.

\vspace{-12pt}
\paragraph{Representation shaping via embedding regularization}
In representation learning, training objectives are often designed to shape the latent space toward desirable geometric or relational properties. For instance, contrastive learning structures representations by enforcing similarity between positive pairs and separation between negative pairs~\cite{oord2018representation, chen2020simple, liu2024cuts, sun2025geometry, he2019momentum, chen2020improved, liu2025diffkillr, liu2025imageflownet, givechian2025immunostruct}; geometry-aware regularization methods promote well-behaved latent structures through explicit geometric constraints~\cite{wang2020understanding, liao2025rnagenscape, sun2024geometry, verma2018manifold}; REPA~\cite{yu2024representation} improves generative quality by aligning representations learned by generative models with those of pretrained understanding models. Closest in spirit, the ``diffuse and disperse'' framework~\cite{diffuse_and_disperse} and \citet{tokarchuk2025keep} introduce dispersive objectives. In a concurrent work, Li et al.~\cite{li2026predictive} study model-level dispersion among hidden representations in language models, and train with objectives that increase dispersion to reduce language model perplexity. Our work instead centers on layer-wise embedding condensation as depth and scale evolve and evaluates dispersion-aware training for models on broad understanding benchmarks.

\vspace{-6pt}
\section{Conclusion}
We presented an empirical study of embedding geometry in Transformer models and identified embedding condensation as a pervasive phenomenon that disproportionately affects smaller models. By introducing a dispersion-aware training objective, we showed that embedding geometry can be directly regulated during training, leading to more diverse embedding vector directions and consistent performance improvements in small language models without scaling up the model size. Our findings suggest that geometric properties of representations are an important and previously underexplored axis for understanding and improving Transformer models. We hope that our work will motivate further investigation into geometry-aware objectives as a complementary approach to scaling in areas including but not limited to language modeling.

\section*{Impact Statement}
This paper studies the geometry of token representations in Transformer-based language models and introduces a dispersion-based regularization objective to counteract embedding condensation. The proposed method is a training-time modification that improves generalization in small language models without changing model architecture or increasing parameter count.

As a representation-level technique, this work does not introduce new application domains or deployment mechanisms. Its potential impact is indirect and mediated by downstream use of language models trained with dispersion-aware objectives. While improved generalization in smaller models may reduce reliance on larger models, the societal implications of such improvements depend on the specific tasks, data, and deployment contexts chosen by future users.

We do not anticipate new ethical risks arising specifically from the proposed loss formulation beyond those already associated with language model training and evaluation.

\section*{Acknowledgements}

S.K. is funded by the NIH (NIGMSR01GM135929, R01GM130847), NSF CAREER award IIS-2047856, NSF IIS-2403317, NSF DMS-2327211 and NSF CISE-2403317. S.K is also funded by the Sloan Fellowship FG-2021-15883, the Novo Nordisk grant GR112933.

This manuscript was co-authored by Oak Ridge National Laboratory (ORNL), operated by UT-Battelle, LLC under Contract No.~DE-AC05-00OR22725 with the U.S. Department of Energy. Any subjective views or opinions expressed in this paper do not necessarily represent those of the U.S. Department of Energy or the United States Government.

\bibliography{references}
\bibliographystyle{icml2026}

\renewcommand{\thefigure}{S\arabic{figure}}
\renewcommand{\theHfigure}{S\arabic{figure}}
\setcounter{figure}{0}
\renewcommand{\thetable}{S\arabic{table}}
\renewcommand{\theHtable}{S\arabic{table}}
\setcounter{table}{0}

\clearpage
\newpage

\renewcommand\appendixpagename{\centering\noindent\rule{\textwidth}{2pt} \LARGE Technical Appendices \\ \normalsize \noindent\rule{\textwidth}{1pt}}

\begin{appendices}

\appendix
\onecolumn
\appendixpage

\vspace{24pt}

\Large \textbf{Table of Contents} \normalsize
\startcontents[sections]
\printcontents[sections]{l}{1}{\setcounter{tocdepth}{2}}
\vspace{36pt}

\clearpage
\newpage
\section{Pseudocode for Dispersion Loss}
\label{sec:pseudo_code}

\definecolor{codeblue}{rgb}{0.25,0.5,0.5}
\definecolor{codekw}{rgb}{0.85, 0.18, 0.50}
\definecolor{codesign}{RGB}{0, 0, 255}
\definecolor{codefunc}{rgb}{0.85, 0.18, 0.50}

\lstdefinelanguage{PythonFuncColor}{
  language=Python,
  keywordstyle=\color{codekw}\bfseries,
  commentstyle=\color{codeblue},
  stringstyle=\color{orange},
  showstringspaces=false,
  basicstyle=\ttfamily\small,
  columns=fullflexible,
  breaklines=true,
  literate=
    {*}{{\color{codesign}*\ }}1
    {-}{{\color{codesign}-\ }}1
    {+}{{\color{codesign}+\ }}1
    {@}{{\color{codesign}@\ }}1
    {bool}{{\color{codefunc}bool}}1
    {mean}{{\color{codefunc}mean}}1
    {exp}{{\color{codefunc}exp}}1
    {log}{{\color{codefunc}log}}1
    {logsumexp}{{\color{codefunc}logsumexp}}1
    {clamp}{{\color{codefunc}clamp}}1
    {arccos}{{\color{codefunc}arccos}}1
    {transpose}{{\color{codefunc}transpose}}1
    {linalg_norm}{{\color{codefunc}linalg\_norm}}1
    {logit}{logit}1
    {log-sum-exp}{log-sum-exp\ }1
}

\begin{algorithm}[!h]
\caption{Dispersion Loss}
\label{alg:dispersion_loss}
\begin{lstlisting}[language=PythonFuncColor]
# z: Token embeddings
# tau: Temperature
# eps: Numerical stability constant
# thr: Clamp threshold

# [batch size, sequence length, feature dimension]
B, L, F = z.shape

# normalize embeddings along feature dimension
z_norm = z / (linalg_norm(z, dim=2, keepdim=True) + eps)

# cosine similarity matrix with shape [B, L, L]
cossim = z_norm @ transpose(z_norm, dim1=1, dim2=2)

# Clamp to avoid negative inf gradient at the two extrema.
# Also saturate the gradient at the boundary to avoid large gradient.
cossim_clmp = clamp(cossim, -1 + thr, 1 - thr)

# Clamp gives 0 gradient beyond the boundary.
# We force same gradient as the boundary instead.
cossim_clmp = cossim + (cossim_clmp - cossim).detach()

# distance matrix with shape [B, L, L]
D = arccos(cossim_clmp) / pi

# mask out diagonal entries
mask = eye(L).bool()
D = D[:, ~mask]

# compute logit
logit = -D / tau

# log-sum-exp trick for `log(mean(exp(logit)))`
loss = logsumexp(logit + eps, dim=1) - log(L * (L - 1))

# dispersion loss
loss = mean(loss)
\end{lstlisting}
\end{algorithm}

\clearpage
\newpage
\section{Experimental Settings}

\subsection{Settings and hyperparameters for training and evaluation}
\label{sec:setting_and_hyperparams}

The settings and hyperparameters are summarized in Table~\ref{tab:supp_experimental_setting}.

\renewcommand{\arraystretch}{1.2}
\begin{table}[!hb]
\caption{Settings and hyperparameters for training and evaluation.}
\vspace{-4pt}
\centering
\setlength{\tabcolsep}{6pt}
\resizebox{0.95\textwidth}{!}{%
    \begin{tabular}{lcccccc}
    \toprule
    & Highly controlled experiment & \texttt{GPT2} mid-training & \texttt{GPT2-m} mid-training & \texttt{Qwen3-0.6B} mid-training & \texttt{Qwen3-1.7B} mid-training & \texttt{Qwen3-0.6B} pre-training \\
    & (Section~\ref{sec:controlled_experiment}) & (Section~\ref{sec:mid_training}) & (Section~\ref{sec:mid_training}) & (Section~\ref{sec:mid_training}) & (Section~\ref{sec:mid_training}) & (Section~\ref{sec:pre_training}) \\
    \midrule
    \textbf{Training}\\
    Optimizer & \multicolumn{6}{c}{-------------------------------------------------------------------------------~~AdamW~\cite{AdamW}~~-----------------------------------------------------------------------------} \\
    Learning rate scheduler & \multicolumn{5}{c}{------------------------------------------------~~Cosine decay with warmup~\cite{SGDR}~~------------------------------------------------} & Linear decay with warmup \\
    Learning rate & $1 \times 10^{-5}$ & $5 \times 10^{-5}$ & $5 \times 10^{-5}$ & $2 \times 10^{-5}$ & $5 \times 10^{-6}$ & $5 \times 10^{-5}$ \\
    Dataset & \texttt{codelion/fineweb-edu-1B} & \multicolumn{4}{c}{-------------------------------------------~~\texttt{Salesforce/wikitext}~~-------------------------------------------} & \texttt{allenai/c4} \\
    GPU count per job & 1 & 1 & 1 & 1 & 1 & 640 \\
    Training duration (tokens) & 1~B & 200~M & 200~M & 200~M & 200~M & 156~B \\
    Batch size per device & 32 & 32 & 16 & 4 & 2 & 2 \\
    Gradient accumulation step & 16 & 4 & 8 & 8 & 16 & 1 \\
    Context length (tokens) & 1024 & 1024 & 1024 & 4096 & 4096 & 4096 \\
    Effective batch size (sequences) & 512 & 128 & 128 & 32 & 32 & 1280 \\
    Effective batch size (tokens) & 524288 & 131072 & 131072 & 131072 & 131072 & 5242880 \\
    \midrule
    \textbf{Evaluation}\\
    Zero-shot examples & --- & 0 & 0 & 0 & 0 & 0 \\
    Few-shot examples & --- & 1 & 1 & 5 & 5 & 5 \\
    Max generation length (tokens) & --- & 256 & 256 & 1024 & 1024 & 1024 \\
    \bottomrule
    \end{tabular}
}
\label{tab:supp_experimental_setting}
\end{table}
\renewcommand{\arraystretch}{1.0}

\subsection{Evaluation benchmarks}
\label{sec:benchmark_description}

\texttt{ANLI$_\text{R2}$}\\The Adversarial Natural Language Inference (ANLI)~\cite{ANLI} (\url{https://huggingface.co/datasets/facebook/anli}) is a new large-scale NLI benchmark. The dataset is collected via an iterative, adversarial human-and-model-in-the-loop procedure.

\texttt{LAMBADA$_\text{openai}$}\\The LAMBADA benchmark~\cite{LAMBADA} (\url{https://huggingface.co/datasets/EleutherAI/lambada_openai}) is a collection of narrative texts sharing the characteristic that human subjects are able to guess their last word if they are exposed to the whole text, but not if they only see the last sentence preceding the target word. To succeed on LAMBADA, computational models cannot simply rely on local context, but must be able to keep track of information in the broader discourse.

\texttt{OpenbookQA}\\The OpenBookQA benchmark~\cite{OpenbookQA} (\url{https://huggingface.co/datasets/allenai/openbookqa}) is a question-answering dataset that contains questions that require multi-step reasoning, use of additional common and commonsense knowledge, and rich text comprehension.

\texttt{PIQA}\\The PIQA benchmark~\cite{PIQA} (\url{https://huggingface.co/datasets/ybisk/piqa}) introduces the task of physical commonsense reasoning, a major challenge on the road to true AI-completeness, including robots that interact with the world and understand natural language.

\texttt{TruthfulQA}\\The TruthfulQA benchmark~\cite{TruthfulQA} (\url{https://huggingface.co/datasets/domenicrosati/TruthfulQA}) is a benchmark to measure whether a language model is truthful in generating answers to questions. The benchmark comprises 817 questions that span 38 categories, including health, law, finance and politics.

\texttt{WinoGrande}\\The WinoGrande benchmark~\cite{WinoGrande} (\url{https://huggingface.co/datasets/allenai/winogrande}) is a collection of 44k problems formulated as a fill-in-a-blank task with binary options.

\texttt{ARC$_\text{easy}$} and \texttt{ARC$_\text{challenge}$}\\The ARC benchmark~\cite{ARC} (\url{https://huggingface.co/datasets/allenai/ai2_arc}) consists of 7,787 genuine grade-school level, multiple-choice science questions, assembled to encourage research in advanced question-answering. The dataset is partitioned into a Challenge Set and an Easy Set, where the former contains only questions answered incorrectly by both a retrieval-based algorithm and a word co-occurrence algorithm.

\texttt{MedMCQA}\\The MedMCQA benchmark~\cite{MedMCQA} (\url{https://huggingface.co/datasets/openlifescienceai/medmcqa}) is a large-scale, Multiple-Choice Question Answering (MCQA) dataset designed to address real-world medical entrance exam questions. It contains more than 194k high-quality AIIMS \& NEET PG entrance exam MCQs covering 2.4k healthcare topics and 21 medical subjects are collected with an average token length of 12.77 and high topical diversity.

\texttt{MMLU}\\The MMLU benchmark~\cite{MMLU, MMLU_ETHICS} (\url{https://huggingface.co/datasets/cais/mmlu})
is a multitask dataset consisting of multiple-choice questions on 57 tasks. It spans subjects in the humanities, social sciences, hard sciences, and other areas that are important for some people to learn.

\clearpage
\newpage
\section{Additional Results on Embedding Condensation}

\subsection{Embedding condensation results on \texttt{wikitext-103}}

We found consistent trends on embedding condensation across the following model families: \texttt{GPT2}~\cite{GPT2}, \texttt{Qwen1}~\cite{Qwen1}, \texttt{Qwen2.5}~\cite{Qwen2.5}, \texttt{Qwen3}~\cite{Qwen3} and \texttt{Bloom}~\cite{Bloom}, as shown in Figure~\ref{fig:supp_observation}. This figure shows the raw results we used in Figures~\ref{fig:observation}.

\begin{figure*}[!th]
\centering
\includegraphics[width=\linewidth]{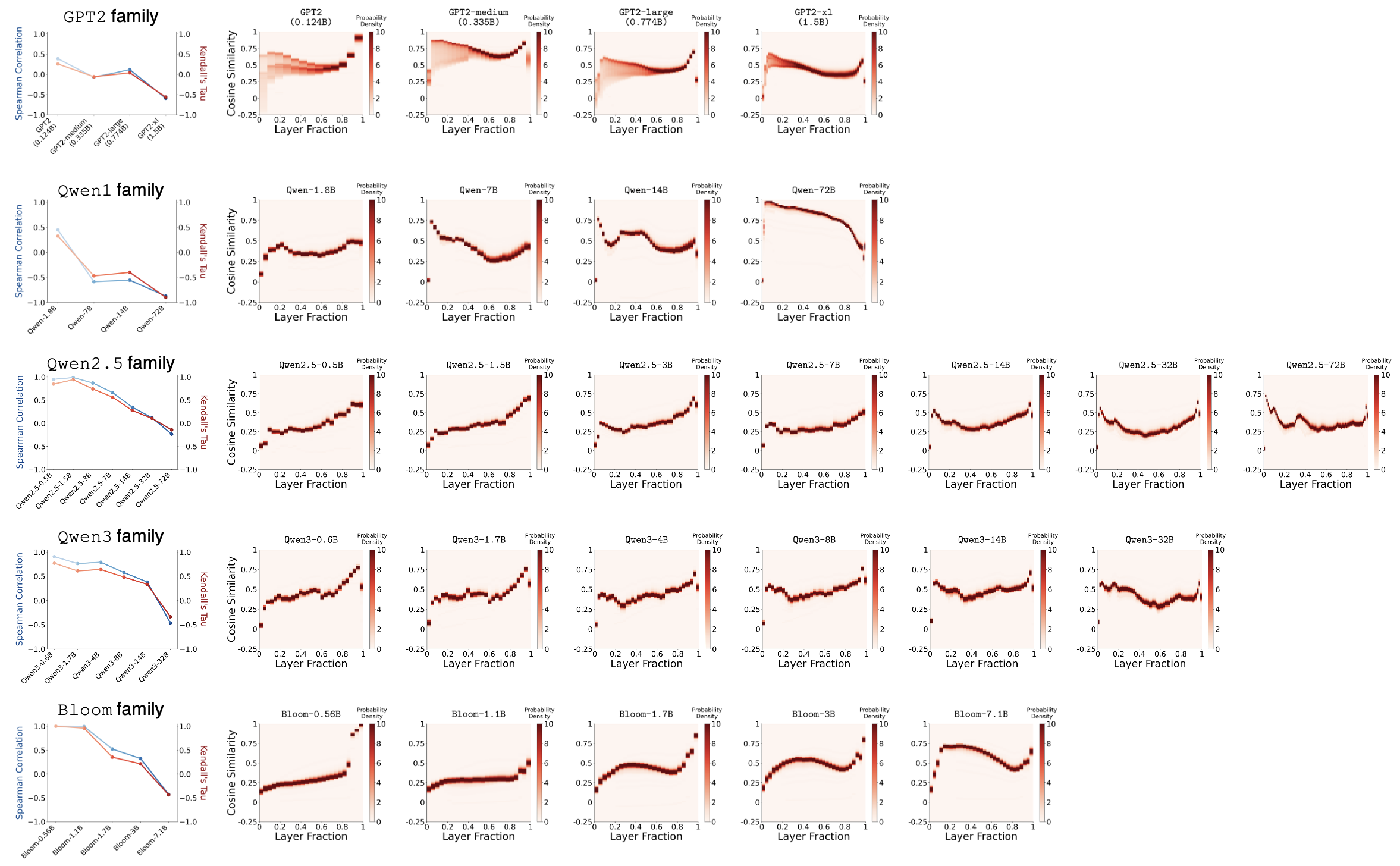}
\caption{Additional quantitative and qualitative evaluations on \texttt{GPT2}, \texttt{Qwen1}, \texttt{Qwen2.5}, \texttt{Qwen3} and \texttt{Bloom} families all demonstrate consistent trends that \textbf{within each model family, larger models are less susceptible to the embedding condensation phenomenon}.}
\label{fig:supp_observation}
\end{figure*}

\clearpage
\newpage
\subsection{Embedding condensation results on \texttt{pubmed\_qa}, \texttt{imdb}, and \texttt{squad}}
\label{sec:condensation_other_datasets}

\begin{figure*}[!th]
\centering
\includegraphics[width=\linewidth]{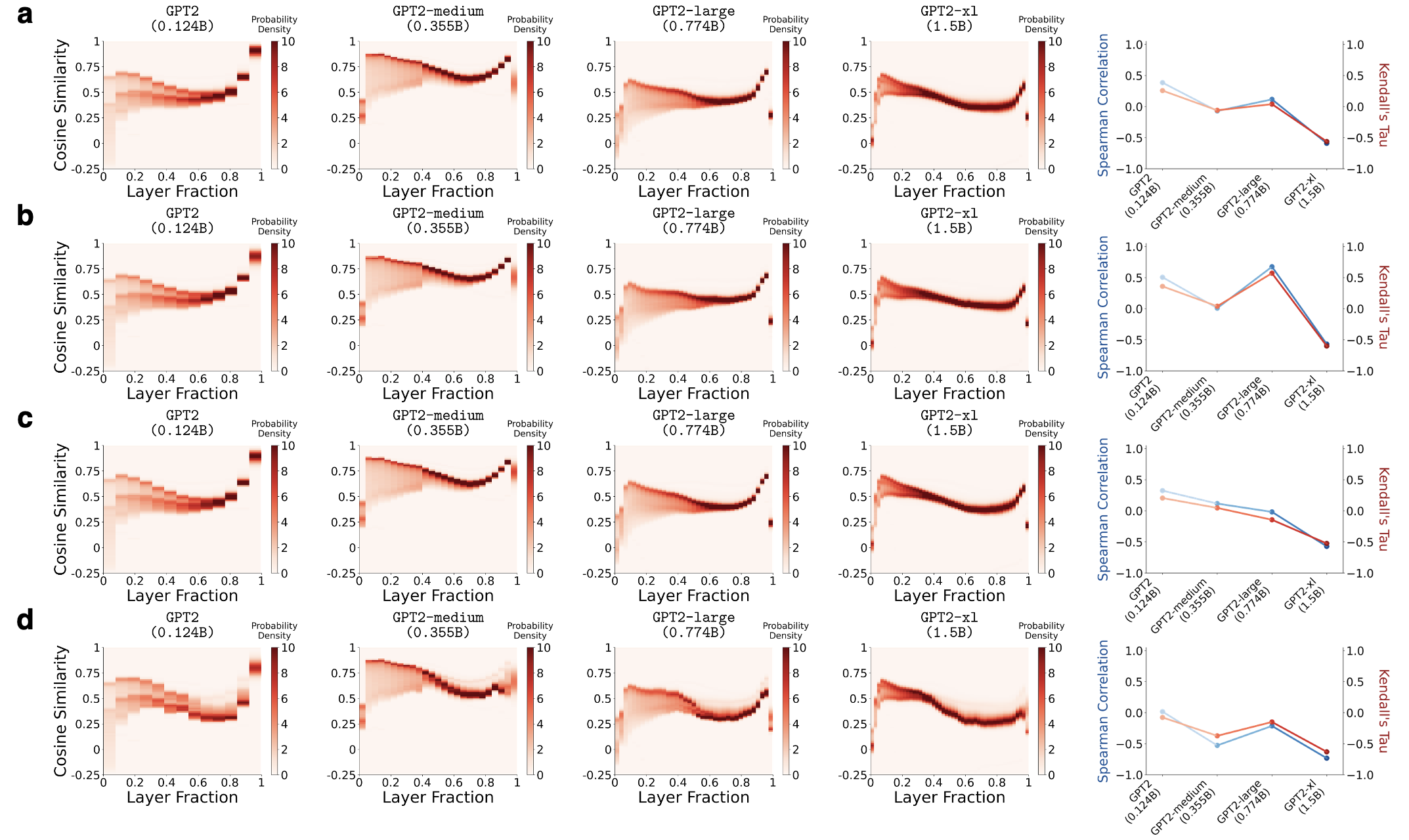}
\caption{The embedding condensation effect is consistent regardless of the input text dataset. Results are shown for four datasets, namely \textbf{(a)} \texttt{wikitext}, \textbf{(b)} \texttt{pubmed\_qa}, \textbf{(c)} \texttt{imdb}, and \textbf{(d)} \texttt{squad}.}
\label{fig:supp_change_dataset}
\end{figure*}

\vspace{24pt}

\subsection{Embedding condensation results on \texttt{wikitext-103}, evaluated using average \texttt{cossim} of last few layers}
\label{sec:condensation_average_cossim}

\begin{figure*}[!th]
\centering
\includegraphics[width=\linewidth]{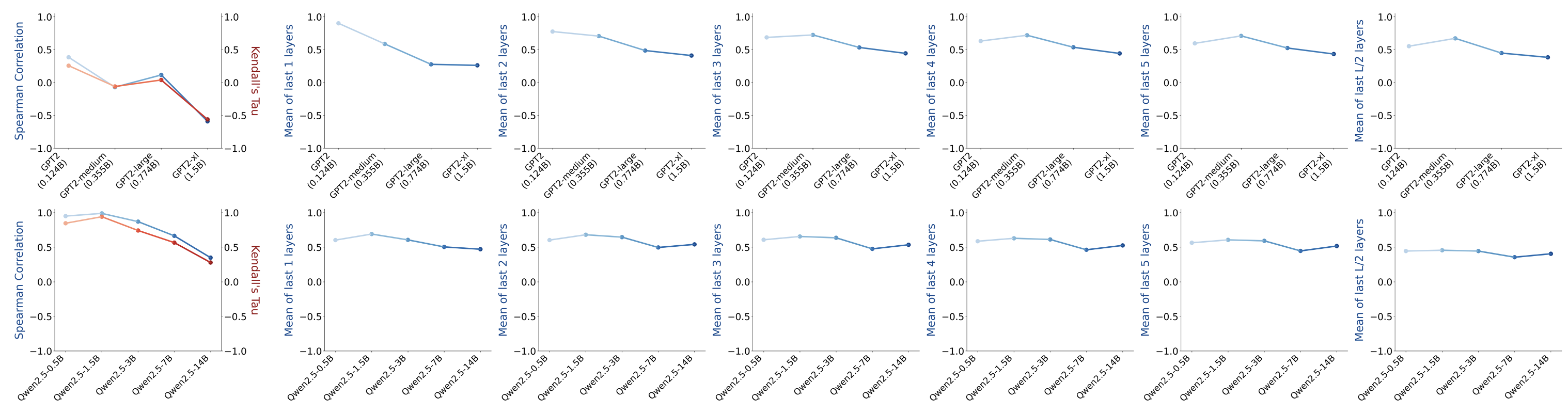}
\caption{The average cosine similarity of the last few layers is also an effective metric that captures the embedding condensation phenomenon. However, the Spearman correlation and Kendall's Tau that we use in this work are able to show the trend more clearly.}
\label{fig:supp_change_metric}
\end{figure*}

\clearpage
\newpage
\section{Embedding condensation and embedding dimension}
\label{sec:dimension}

It is widely known that when the dimension increases, random vectors are more likely to be orthogonal to each other~\cite{random_angle_high_dim}. As a result, we ask the following question.

\begin{tcolorbox}[
    colback=natureblue!10,
    boxrule=1pt,
    arc=8pt,
    left=4pt,
    right=4pt,
    top=0pt,
    bottom=0pt
]
\textbf{Q:} To what extent does increased embedding dimensionality alone account for the reduced embedding condensation observed in larger models, independent of training dynamics?
\end{tcolorbox}

To answer this question, we provide the following theoretical results on the expected values of cosine similarity between two vectors $x, y \in \mathbb{R}^d$ if the dimension increases from $d$ to $D$. We consider two idealized mechanisms for increasing embedding dimension, which serve as geometric reference cases rather than models of actual training dynamics.

Note that the cosine similarity between $x$ and $y$ at dimension $d$ is $\mathrm{cossim}_d(x,y) = \frac{x^\top y}{\lVert x \rVert \cdot \lVert y \rVert}$. Assume $\mathrm{cossim}_d(x,y) \geq 0$.
\begin{enumerate}
    \item If the dimension growth is achieved by repeating vector entries, the cosine similarity stays the same~\ref{prop:repeat_invariant}.
    \item If the dimension growth is achieved by padding random entries from the standard normal distribution, the expected cosine similarity is $\mathrm{cossim}_d(x,y) \cdot \alpha(\lVert x \rVert) \cdot \alpha(\lVert y \rVert)$, which is strictly between
    \begin{align}
    \nonumber
    \mathrm{cossim}_d(x,y) \frac{\lVert x \rVert \cdot \lVert y \rVert}{\sqrt{\lVert x \rVert^2 + D-d}\sqrt{\lVert y \rVert^2 + D-d}}
    &= \frac{x^\top y}{\lVert x \rVert \cdot \lVert y \rVert} \frac{\lVert x \rVert \cdot \lVert y \rVert}{\sqrt{\lVert x \rVert^2 + D-d}\sqrt{\lVert y \rVert^2 + D-d}}\\
    \nonumber
    &= \frac{x^\top y}{\sqrt{\lVert x \rVert^2 + D-d}\sqrt{\lVert y \rVert^2 + D-d}}
    \end{align}
    and
    \begin{align}
    \nonumber
    \mathrm{cossim}_d(x,y) \frac{\lVert x \rVert \cdot \lVert y \rVert}{\sqrt{\lVert x \rVert^2 + D-d-1}\sqrt{\lVert y \rVert^2 + D-d-1}} 
    &= \frac{x^\top y}{\lVert x \rVert \cdot \lVert y \rVert} \frac{\lVert x \rVert \cdot \lVert y \rVert}{\sqrt{\lVert x \rVert^2 + D-d-1}\sqrt{\lVert y \rVert^2 + D-d-1}}\\
    \nonumber
    & = \frac{x^\top y}{\sqrt{\lVert x \rVert^2 + D-d-1}\sqrt{\lVert y \rVert^2 + D-d-1}}
    \end{align}
    If $x$ and $y$ are unit vectors, this implies that the new cosine similarity is between $\frac{\mathrm{cossim}_d(x,y)}{D - d + 1}$ and $\frac{\mathrm{cossim}_d(x,y)}{D - d}$.
\end{enumerate}

Taking the \texttt{GPT2} family as an example, the smallest model \texttt{GPT2} has an embedding dimension of 768, while the largest model \texttt{GPT2-xl} has an embedding dimension of 1600. The theoretical results above imply that if the average cosine similarity in \texttt{GPT2} is $\mathrm{cossim}(\texttt{GPT2})$, then the corresponding value in \texttt{GPT2-xl} would be
\begin{enumerate}
    \item equal to $\mathrm{cossim}(\texttt{GPT2})$ if assuming repeating entries, or
    \item between $\frac{\mathrm{cossim}(\texttt{GPT2})}{1600-768+1}=\frac{\mathrm{cossim}(\texttt{GPT2})}{833}$ and $\frac{\mathrm{cossim}(\texttt{GPT2})}{1600-768}=\frac{\mathrm{cossim}(\texttt{GPT2})}{832}$ if assuming isotropic Gaussian noise entries.
\end{enumerate}

\begin{tcolorbox}[
    colback=natureblue!10,
    boxrule=1pt,
    arc=8pt,
    left=4pt,
    right=4pt,
    top=0pt,
    bottom=0pt
]
\textbf{A:} In high-dimensional spaces, random vectors are nearly orthogonal, implying that increased embedding dimensionality defines a geometric upper bound on achievable dispersion under isotropy. However, this effect reflects a property of isotropic randomness rather than a guarantee of how trained representations utilize the space.

\vspace{8pt}

\textbf{Implication:} While larger models benefit from a higher-dimensional geometric baseline, embedding dimensionality alone does not ensure well-dispersed representations. This observation raises the possibility that explicit dispersion regularization could further improve embedding geometry even in large models, which we leave for future investigation.
\end{tcolorbox}

\clearpage
\newpage
\section{Proofs of Propositions}

\begin{proposition}
\label{prop:repeat_invariant}
Let $x,y \in \mathbb{R}^d$ be nonzero vectors. Let $D = k d$ for some integer $k \ge 1$, and define the repeated vectors $\tilde x = (x,x,\ldots,x) \in \mathbb{R}^{D}$ and $\tilde y = (y,y,\ldots,y) \in \mathbb{R}^{D}$. Then $\mathrm{cossim}_{D}(\tilde x,\tilde y) = \mathrm{cossim}_{d}(x,y)$. Consequently, if $(x,y)$ are random and satisfy $\mathbb{E}[\mathrm{cossim}_{d}(x,y)] = c$, then $\mathbb{E}[\mathrm{cossim}_{D}(\tilde x,\tilde y)] = c$.
\end{proposition}

\begin{proof}
Let $x, y \in \mathbb{R}^d$ be fixed. Then each of the repeated vectors $\tilde{x}, \tilde{y} \in \mathbb{R}^{D}$ consists of $k$ concatenated copies of $x$ and $y$, respectively.

We compute the inner product:
\[
\langle \tilde{x}, \tilde{y} \rangle
= \sum_{i=1}^{k} \langle x, y \rangle = k \cdot \langle x, y \rangle.
\]
Next, compute the norms:
\[
\|\tilde{x}\|^2 = \sum_{i=1}^{k} \|x\|^2 = k \cdot \|x\|^2, \quad
\|\tilde{y}\|^2 = \sum_{i=1}^{k} \|y\|^2 = k \cdot \|y\|^2.
\]
Thus:
\[
\|\tilde{x}\| = \sqrt{k} \cdot \|x\|, \quad
\|\tilde{y}\| = \sqrt{k} \cdot \|y\|.
\]

Plugging into the cosine similarity:
\[
\mathrm{cossim}_D(\tilde{x}, \tilde{y})
= \frac{\langle \tilde{x}, \tilde{y} \rangle}{\|\tilde{x}\| \cdot \|\tilde{y}\|}
= \frac{k \cdot \langle x, y \rangle}{(\sqrt{k} \cdot \|x\|)(\sqrt{k} \cdot \|y\|)}
= \frac{\langle x, y \rangle}{\|x\| \cdot \|y\|}
= \mathrm{cossim}_d(x, y).
\]

The identity thus holds pointwise. If $(x, y)$ are random vectors with $\mathbb{E}[\mathrm{cossim}_d(x, y)] = c$, then by linearity of expectation:
\[
\mathbb{E}[\mathrm{cossim}_D(\tilde{x}, \tilde{y})]
= \mathbb{E}[\mathrm{cossim}_d(x, y)] = c.
\]
\end{proof}


\begin{proposition}
\label{prop:gaussian_padding_shrink}
Let $x, y \in \mathbb{R}^d$ be nonzero vectors, let $m \ge 1$, and define $D = d + m$. Construct padded vectors $X = (x, \varepsilon)$ and $Y = (y, \eta) \in \mathbb{R}^D$, where $\varepsilon, \eta \sim \mathcal{N}(0, I_m)$ are independent standard Gaussian noise vectors, also independent of $(x, y)$. Define:
\[
\alpha(r) := \mathbb{E}_{U \sim \chi^2_m} \left[ \frac{r}{\sqrt{r^2 + U}} \right].
\]
Then:
\begin{align}
\mathbb{E}[\mathrm{cossim}_D(X,Y)]
&= \mathrm{cossim}_d(x,y) \cdot \alpha(\|x\|) \cdot \alpha(\|y\|). \label{eq:main_identity}
\end{align}
Moreover, for all $r > 0$, the function $\alpha(r)$ satisfies the strict bounds:
\begin{align}
\frac{r}{\sqrt{r^2 + m}} &< \alpha(r) < \frac{r}{\sqrt{r^2 + m - 1}}, \label{eq:alpha_bounds}
\end{align}
with both inequalities strict. As a consequence, the expected cosine similarity after Gaussian padding obeys:
\begin{align}
\mathrm{cossim}_d(x,y) \cdot
\frac{\|x\| \|y\|}{\sqrt{(\|x\|^2 + m)(\|y\|^2 + m)}}
&< \mathbb{E}[\mathrm{cossim}_D(X,Y)] \\
&< \mathrm{cossim}_d(x,y) \cdot
\frac{\|x\| \|y\|}{\sqrt{(\|x\|^2 + m - 1)(\|y\|^2 + m - 1)}}, \label{eq:cossim_bounds}
\end{align}
again with strict inequalities.
\end{proposition}

\begin{proof}
We first compute the cosine similarity in $D$ dimensions. By definition:
\[
\mathrm{cossim}_D(X,Y) = \frac{\langle X, Y \rangle}{\|X\| \cdot \|Y\|}.
\]
The inner product expands as:
\[
\langle X, Y \rangle = \langle x, y \rangle + \langle \varepsilon, \eta \rangle.
\]
Since $\varepsilon$ and $\eta$ are independent standard Gaussian vectors in $\mathbb{R}^m$, their inner product $\langle \varepsilon, \eta \rangle$ has mean zero. Specifically:
\[
\mathbb{E}[\langle \varepsilon, \eta \rangle] = \sum_{i=1}^m \mathbb{E}[\varepsilon_i \eta_i] = 0,
\]
since each $\varepsilon_i$ and $\eta_i$ are independent with mean zero.

For the denominator, we write:
\[
\|X\|^2 = \|x\|^2 + \|\varepsilon\|^2, \quad \|Y\|^2 = \|y\|^2 + \|\eta\|^2.
\]
Denote $U = \|\varepsilon\|^2$ and $V = \|\eta\|^2$. Since $\varepsilon, \eta \sim \mathcal{N}(0, I_m)$, we know:
\[
U, V \sim \chi^2_m, \quad \text{independently}.
\]

Taking expectation of the cosine similarity:
\[
\mathbb{E}[\mathrm{cossim}_D(X,Y)] = \mathbb{E} \left[ \frac{\langle x, y \rangle + \langle \varepsilon, \eta \rangle}{\sqrt{\|x\|^2 + U} \cdot \sqrt{\|y\|^2 + V}} \right].
\]
By linearity of expectation and the independence of $\varepsilon, \eta$, we obtain:
\[
\mathbb{E}[\mathrm{cossim}_D(X,Y)] = \langle x, y \rangle \cdot \mathbb{E} \left[ \frac{1}{\sqrt{\|x\|^2 + U} \cdot \sqrt{\|y\|^2 + V}} \right].
\]
Now observe that $\langle x, y \rangle = \mathrm{cossim}_d(x, y) \cdot \|x\| \cdot \|y\|$. So we substitute:
\[
\mathbb{E}[\mathrm{cossim}_D(X,Y)] = \mathrm{cossim}_d(x, y) \cdot \|x\| \cdot \|y\| \cdot \mathbb{E} \left[ \frac{1}{\sqrt{\|x\|^2 + U} \cdot \sqrt{\|y\|^2 + V}} \right].
\]
Since $U$ and $V$ are independent, the expectation factorizes:
\[
\mathbb{E}[\mathrm{cossim}_D(X,Y)] = \mathrm{cossim}_d(x, y) \cdot \left( \mathbb{E} \left[ \frac{\|x\|}{\sqrt{\|x\|^2 + U}} \right] \cdot \mathbb{E} \left[ \frac{\|y\|}{\sqrt{\|y\|^2 + V}} \right] \right).
\]
This yields the desired expression with
\[
\alpha(r) = \mathbb{E}_{U \sim \chi^2_m} \left[ \frac{r}{\sqrt{r^2 + U}} \right].
\]

To prove the lower bound, we consider the function:
\[
f(u) = \frac{r}{\sqrt{r^2 + u}}.
\]
We compute its second derivative:
\[
f''(u) = \frac{3r}{4} (r^2 + u)^{-5/2} > 0 \quad \text{for all } u > 0.
\]
Hence, $f$ is strictly convex on $(0, \infty)$.
Applying Jensen's inequality:
\[
\alpha(r) = \mathbb{E}[f(U)] \ge f(\mathbb{E}[U]) = \frac{r}{\sqrt{r^2 + \mathbb{E}[U]}} = \frac{r}{\sqrt{r^2 + m}}.
\]
Because $f$ is strictly convex and $U$ is not constant (since $\operatorname{Var}(U) = 2m > 0$), equality cannot occur. Therefore the inequality is strict:
\[
\alpha(r) > \frac{r}{\sqrt{r^2 + m}}.
\]

To establish the upper bound, observe that $f(u)$ is strictly decreasing:
\[
f'(u) = -\frac{r}{2}(r^2 + u)^{-3/2} < 0.
\]
This implies that for any constant $a < \mathbb{E}[U]$, we have $f(U) > f(a)$ with positive probability and $f(U) < f(a)$ with positive probability, since $U \sim \chi^2_m$ is supported on $(0, \infty)$ and not almost surely equal to any fixed value.

Choosing $a = m - 1 < m$, and noting that $f(U) \le f(m - 1)$ almost surely with strict inequality on a set of positive measure, we conclude:
\[
\alpha(r) = \mathbb{E}[f(U)] < f(m - 1) = \frac{r}{\sqrt{r^2 + m - 1}}.
\]

This completes the proof of both strict bounds.
\end{proof}


\end{appendices}

\end{document}

%% file: math_commands.tex
\usepackage{amsmath,amsfonts,bm}

\def\eqref#1{equation~\ref{#1}}

\def\1{\bm{1}}

\DeclareMathAlphabet{\mathsfit}{\encodingdefault}{\sfdefault}{m}{sl}
\SetMathAlphabet{\mathsfit}{bold}{\encodingdefault}{\sfdefault}{bx}{n}

